\documentclass[11pt]{article}
\usepackage{pkg}
\usepackage{wrapfig}
\usepackage[toc,page]{appendix}
\PassOptionsToPackage{hyphens}{url}
\usepackage[hypertexnames=false]{hyperref}
\usepackage{color}
\usepackage[dvipsnames]{xcolor}
\usepackage{latexsym}
\usepackage{dsfont}
\usepackage{appendix}
\usepackage{amssymb}
\usepackage{subcaption}

\usepackage[margin=1in]{geometry}
\usepackage{float}
\usepackage{algorithm,algorithmic}
\usepackage[sort&compress,numbers]{natbib}
\usepackage{amsmath, amsfonts,amssymb,euscript,array,amsfonts,mathrsfs}
\usepackage{mathabx}
\usepackage{enumerate}
\usepackage[dvipsnames]{xcolor}
\usepackage{amsthm}

\usepackage[T1]{fontenc}
\usepackage[utf8]{inputenc}
\usepackage{hhline}
\usepackage{graphicx}
\usepackage{verbatim}
\usepackage{eurosym}
\usepackage{dsfont}
\usepackage{mathtools}
\usepackage{soul}

\allowdisplaybreaks

\DeclareMathOperator*{\argmin}{arg\,min}

\DeclareMathOperator{\dd}{{\rm d}}

\makeatletter
\@addtoreset{equation}{section}

\newcommand*{\rom}[1]{\expandafter\@slowromancap\romannumeral #1@}
\makeatother
\newcommand{\vertiii}[1]{{\left\vert\kern-0.25ex\left\vert\kern-0.25ex\left\vert #1 
		\right\vert\kern-0.25ex\right\vert\kern-0.25ex\right\vert}}

\newcommand{\Renyuan}[1]{{\textcolor{cyan}{{\textsf{Renyuan: #1}}}}}
\newcommand{\Ying}[1]{{\textcolor{ForestGreen}{{\textsf{Ying: #1}}}}}

\newcommand{\ba}{\begin{array}}
	\newcommand{\ea}{\end{array}}
\newcommand{\be}{\begin{equation}}
	\newcommand{\ee}{\end{equation}}
\newcommand{\eeqn}{\end{equation}}
\newcommand{\bea}{\begin{eqnarray}}
\newcommand{\eea}{\end{eqnarray}}
\newcommand{\beaa}{\begin{eqnarray*}}
\newcommand{\eeaa}{\end{eqnarray*}}

\def\a{\alpha}

\def\d{\delta}
\def\e{\varepsilon}

\def\k{\kappa}
\def\l{\lambda}
\def\m{\mu}
\def\n{\nu}
\def\si{\sigma}

\def\f{\varphi}

\def\o{\omega}
\def\h{\widehat}

\def\G{\Gamma}
\def\D{\Delta}

\def\O{\Omega}

\def\cA{{\cal A}}

\def\cF{{\cal F}}

\def\cK{{\cal K}}

\def\cS{{\cal S}}
\def\cT{{\cal T}}

\def\hC{\mathbb{C}}

\def\hE{\mathbb{E}}
\def\hF{\mathbb{F}}

\def\hI{\mathbb{I}}

\def\hL{\mathbb{L}}

\def\hN{\mathbb{N}}

\def\hP{\mathbb{P}}
\def\hQ{\mathbb{Q}}
\def\hR{\mathbb{R}}

\def\hX{\mathbb{X}}

\def\sB{\mathscr{B}}

\def\sD{\mathscr{D}}
\def\sE{\mathscr{E}}

\def\scL{\mathscr{L}}
\def\sM{\mathscr{M}}

\def\sP{\mathscr{P}}

\def\no{\noindent}

\def\ss{\smallskip}
\def\ms{\medskip}
\def\bs{\bigskip}
\def\q{\quad}
\def\qq{\qquad}

\def\pa{\partial}
\def\cd{\cdot}

\def\td{\nabla}

\def\limsup{\mathop{\overline{\rm lim}}}
\def\liminf{\mathop{\underline{\rm lim}}}

\def\limw{\mathop{\buildrel w\over\rightharpoonup}}

\def\pa{\partial}
\def\h{\widehat}
\def\wt{\widetilde}
 \def\cd{\cdot}

\def\1{{\bf 1}}
\def\neg{\negthinspace}

\def\lf[{{[\![}}
\def\Blf[{{\Big[\!\!\Big[}}
\def\ri]{{]\!]}}
\def\Bri]{{\Big]\!\!\Big]}}
\def\reff{\eqref}
\def\dis{\displaystyle}

\def\wt{\widetilde}

\def\bc{{\bf c}}

\newtheorem{Theorem}{Theorem}[section]
\newtheorem{Definition}[Theorem]{Definition}
\newtheorem{Proposition}[Theorem]{Proposition}
\newtheorem{Assumption}[Theorem]{Assumption}
\newtheorem{Lemma}[Theorem]{Lemma}

\newtheorem{Remark}[Theorem]{Remark}
\newtheorem{Example}[Theorem]{Example}

\newtheorem{Problem}[Theorem]{Problem}

\def\mi{{\mu_{\rm ini}}}
\def\mt{{\mu_{\rm tar}}}

\begin{document}
\title{Schr\"odinger bridge for generative AI: Soft-constrained formulation and convergence analysis\footnote{We thank helpful comments and discussions from Beatrice Acciaio, Rama Cont, Huyên Pham,  Xin Zhang, and Xun Yu Zhou.}\\
	\vspace{5pt}
	{\it \large In celebration of Prof. Xun Yu Zhou's 60th birthday}
}

\author{Jin Ma\thanks{Department of Mathematics, University of Southern California, Los Angeles, CA, 90089.
		Email: jinma@usc.edu. This author is supported in part by US NSF grants DMS\#2205972 and \#2510403.
	}, ~Ying Tan\thanks{Department of Statistics and Applied Probability, University of California, Santa Barbara, CA 93106. Email: yingtan@ucsb.edu.}, ~and 
	Renyuan Xu\thanks{Department of Management of Science \& Engineering, Stanford University, Stanford, CA 94305. Email: renyuanxu@stanford.edu.  This author is supported in part by the NSF CAREER Award DMS-2524465.}
}
\date{\today}
\maketitle

\begin{abstract}
	Generative AI can be framed as the problem of learning a model that maps simple reference measures into complex data distributions, and it has recently found a strong connection to the classical theory of the Schrödinger bridge problems (SBPs)
	due partly to their common nature of  interpolating between prescribed marginals via entropy-regularized stochastic dynamics. However, the classical SBP enforces hard terminal constraints, which often leads to instability in practical implementations, especially in high-dimensional or data-scarce regimes.
	To address this challenge, we follow the idea of the so-called {\it soft-constrained Schrödinger bridge problem} (SCSBP), in which the terminal constraint is replaced by a general penalty function. This relaxation leads to a more flexible stochastic control formulation of McKean–Vlasov type.

	We establish the existence of optimal solutions for all penalty levels and prove that, as the penalty grows, both the controls and value functions converge to those of the classical SBP at a linear rate. Our analysis builds on Doob's $h$-transform representations, the stability results of Schrödinger potentials, $\Gamma$-convergence, and a novel fixed-point argument that couples an optimization problem over the space of measures with an auxiliary entropic optimal transport problem. These results not only provide the first quantitative convergence guarantees for soft-constrained bridges but also shed light on how penalty regularization enables robust generative modeling, fine-tuning, and transfer learning.
\end{abstract}


\vfill \bs

\no


\bs

\no{\it 2020 AMS Mathematics subject classification:} 60H10, 60J60, 49J21, 68T01, 93E20.

\section{Introduction}
Generative modeling provides a powerful framework for synthesizing data that preserves the statistical structure of real-world samples while introducing controlled variability. Among the most prominent approaches, diffusion models — such as those introduced by \cite{sohl2015deep,song2019generative,song2020score,Ho20} — have achieved remarkable success, underpinning state-of-the-art systems like DALL·E 2 and 3 \cite{ramesh2022hierarchical,betker2023improving}, Stable Diffusion \cite{rombach2022high}, and Sora \cite{openai2024sora}. These models learn to reverse a diffusion process that gradually adds noise to data, enabling the generation of realistic samples from pure noise. Such a structure, namely, transforming a noise distribution into a data distribution, closely mirrors the Schrödinger bridge problem (or dynamic optimal transport), which has recently gained renewed attention as a principled framework for generative modeling due to its structural parallels with diffusion models and its ability to interpolate between distributions in a statistically grounded manner.  

The Schrödinger bridge problem (SBP for short), originally proposed as an entropy-regularized variant of optimal transport, seeks the most likely evolution of a process, subject to a reference diffusion process, that matches prescribed marginal distributions $\mi,\mt$ at two endpoints. Under suitable regularity conditions, the optimally controlled process remains a diffusion but with an additional  drift term added to the reference process. This result has been established through various approaches and levels of generality, with seminal contributions by \cite{fortet1940resolution,beurling1960automorphism,jamison1975markov,follmer2006random,dai1991stochastic}.

The recent generative modeling literature has seen a surge in the use of Schrödinger bridges. In these applications, one typically starts or chooses some distribution $\mi$ that is easy to sample from, and tries to ``learn'' an unknown distribution $\mt$ of a given data set. By numerically approximating the solution to the Schrödinger bridge problem, one can generate unlimited samples (i.e., synthetic data points that resemble the original data set). One such algorithm is presented by De Bortoli et al. \cite{de2021diffusion}  and Vargas et al.~\cite{vargas2021solving}, who approximate the iterative proportional fitting procedure (Deming-Stephan  \cite{deming1940least}),  using score matching with neural networks and maximum likelihood, respectively. 
Concurrently, Wang et al. \cite{wang2021deep} proposed a two-stage method with an auxiliary bridge handling possible non-smooth $\mt$. Some more recent developments include  
\cite{chen2021likelihood,song2022applying,peluchetti2023diffusion,richter2023improved,winkler2023score,hamdouche2023generative,vargas2023transport,shi2024diffusion,uehara2024fine,alouadi2025robust} as well as developing optimal transport techniques for generative AI tasks \cite{montesuma2024recent,benamou2024entropic,loeper2023calibration,xu2020cot,acciaio2024time}.

However, the classical SBP imposes hard terminal constraints on the marginal distributions, which can result in computational difficulties, instability in high-dimensional settings, and limited adaptability when aligning with empirical data in generative tasks. In practice, most numerical schemes for solving the SBP rely on iterative procedures that alternately relax the initial and terminal constraints. These algorithms can exhibit instability, particularly when the two constraints differ significantly, and their convergence guarantees for general target distributions remain an open problem in the literature.

In this paper, in light of Garg et al. \cite{garg2024soft}, we study a soft-constrained Schrödinger bridge problem (SCSBP).
Mathematically, we consider a penalty function $G:\sP_2(\hR^d)\to \hR_+$, satisfying $G(\mu;\mt)=0$ if and only if $\mu=\mt$, 
where $\sP_2(\hR^d)$ denotes {all the probability measures on $\hR^d$ with finite second moment equipped with 2-Wasserstein metric}, and $\mt \in \sP_2(\hR^d)$ is some given ``target'' distribution. For each $k\in\hN$ and a given initial distribution $\mi\in \sP_2(\hR^d)$, we consider the following stochastic control problem with dynamics {on a given probability space $(\Omega,\cF,\hP)$}:
\begin{equation*}
	d X^\alpha_t = \big(b(t,X^\alpha	_t) + \sigma(t) \alpha_t\big)d t + \sigma(t)d W_t, \quad \hP\circ (X_0^\alpha)^{-1} = \mi,	
\end{equation*}
and cost functional 
\begin{equation*}
	J^k(\alpha)=\hE \left[ \frac{1}{2} \int_0^T |\alpha_s|^2 d s  + k G(\mathbb{P}_{X^\alpha_T})\right],
\end{equation*}
where $\mathbb{P}_{X^\alpha_T}$ is the law of $X^\alpha_T$ {under probability $\hP$} and the control $\a$ is chosen from a square-integrable progressively measurable   {\it admissible control} set $\cA$. The goal is to find $V^k: =\inf_{\alpha \in \cA} J^k(\alpha)$ and optimal control $\a^k$, for each $k\in\hN$, and study the limiting behavior of $\{\a^k\}$ and $\{V^k\}$.  
Clearly, the dependence of the terminal cost on $\mathbb{P}_{X^\alpha_T}$ renders this relaxed formulation a {\it non-standard} stochastic control problem, leading to a McKean-Vlasov type of stochastic control. In contrast to Garg et al. \cite{garg2024soft}, which focuses on the case where the penalty 
$G$ is given by the KL divergence and 
$\mu_{\rm ini}$ is a delta measure, we investigate the problem under more general cost functions and initial distributions.

Compared to existing methods using SBP under hard constraints, there are several advantages to considering SCSBP. First, when the KL divergence between $\mt$ and $\mathbb{P}_{X^\alpha_T}$ is infinite, the Schr\"odinger bridge does not admit a solution, whereas SCSBP always does (cf. e.g.,   \cite{garg2024soft}). More importantly,  the penalty parameter $k$ acts as a regularization factor, preventing the algorithm from overfitting to $\mu_{\rm tar}$, which is crucial for certain generative modeling tasks with limited data \cite{gu2023memorization,moon2022fine}. In addition, SCSBP provides a more general framework in generative AI, with applications beyond pre-training in data generation, and can be applied to fine-tuning and transfer learning (see Examples \ref{exmpl:fine-tuning} and \ref{exmpl:transfer-learning}).

\subsection{Outline of the Main Results and Contributions}

The soft-constrained Schrödinger bridge problem (SCSBP) studied in this paper replaces the terminal distribution constraint by a general penalty function, which leads to a McKean–Vlasov type stochastic control problem. The main results include the existence of the optimal solution to the SCSBP at each penalty level $k$, and the convergence of the solutions to the SCSBP to that of the SBP, in terms
of both the optimal policy and the corresponding value function, as $k\to \infty$.

More precisely, we begin with the special case where the initial measure is a delta measure. In this setting, we derive the explicit form of the optimal control policy for SCSBP via Doob’s $h$-transform (Proposition~\ref{existak}) and more importantly, we establish a linear convergence rate for the optimal control (Theorem~\ref{thm:convergence_opt_policy}). To the best of our knowledge, such a rate of convergence is novel in the literature. Moreover, by applying the so-called {\it  early stopping}, we are able to obtain the linear convergence results for the corresponding value functions (Proposition~\ref{ValueRate}) as well as the Wasserstein distance between the target distribution and the output distribution of the SCSBP (Proposition~\ref{prop:3.9}).

The similar results in the case of a general initial distribution is much more involved. Among other things, we  establish and/or extend some recently observed stability results of the solutions to the SBP, as the foundation for a fixed-point argument. A key element in our argument is the continuous dependence (or stability) of a mapping that is well known in the (static) optimal transport literature. Specifically, fix $\mi\in\sP_2(\hR^d)$ and consider any $\mu\in\sP_2(\hR^d)$, it is known (cf. e.g.,  \cite{beurling1960automorphism}) that there exists a unique pair of $\sigma$-finite measures $( \nu_0,\nu_T):=\cT(\mu)$ such that the measure $\pi$ on $\hR^d \times \hR^d$ defined by 
\bea
\label{pi0}
\pi(E)= \int_E p(T, y;0, x) \nu_0(dx)\nu_T(dy), \qq E\in\sB(\hR^d\times \hR^d)
\eea
has the marginals $\mi$ and $\mu$, where $p(\cd, \cd;\cd,\cd)$ is the transition density of a given diffusion process.  Then it turns out that the solution to the SCSBP is the fixed-point of the  mapping $\G:\sP_2(\hR^d)\to \sP_2(\hR^d)$, where
\bea
\label{Gmu}	
\G(\m):= {\argmin}_{\bar\mu\in\sP_2(\hR^d)}\left\{kG(\bar \mu)+\hE_{X\sim \bar \mu}[\log \rho^{\m}(X)]
\right\},
\eea
where $\rho^{\m}$ is the density of the $\sigma$-finite measure $\nu$ in \eqref{pi0} (which may not be a probability measure). 
The successful application of Schauder’s fixed-point theorem on the Wasserstein space relies on several key elements, in particular a crucial continuous dependence result of the mapping $\m\mapsto \rho^\m$, for which we introduce an auxiliary {\it entropic optimal transport problem}, and identify its solution to the measure $\pi$ in (\ref{pi0}). 
By utilizing some important stability results of the corresponding Schr\"{o}dinger potentials (Proposition \ref{stable}), together with some arguments in the spirit of the converse of Scheff\'e's theorem (i.e., the weak convergence in Prohorov metric  vs. the convergence in densities) as well as the so-called $\G$-convergence of the minimizers of the optimization problem (\ref{Gmu}), we are able to verify the required properties so the mapping $\G$ has a fixed-point (Theorem~\ref{thm:sE}).
As a direct consequence, we establish that the optimal control of the SCSBP converges linearly with respect to the penalty parameter $k$ (Theorem~\ref{thm:convergence_rate_general_initial}). We believe that such a fixed-point perspective is novel in the literature, as it not only offers a constructive framework for characterizing solutions in the general case but also yields insights into how the penalty parameter affects the convergence rate.

\subsection{Closely Related Literature}

Our general formulation is largely inspired by Garg et al. \cite{garg2024soft}, which investigates the SCSBP with the KL divergence as the penalty function $G$. Within that framework, the authors established the asymptotic convergence of the optimal policies as the penalty parameter $k$ tends to infinity, under the assumption that the initial measure is a delta measure.  However, the use of KL divergence presents practical difficulties: if the model distribution $\mu$ assigns zero probability to any region where the data distribution $\mu_{\rm tar}$ has positive mass, then {$D_\mathrm{KL}(\mu_{\rm tar}\|\mu) = +\infty$}, rendering the divergence ill-posed under support mismatch \cite{chen2024regularized,kong2025composite}. Moreover, a delta initial measure is rarely employed in generative tasks, as it lacks the diversity and randomness required for effective training. Finally, no convergence rate is quantified therein.

On a technical level, our formulation is closely related to Hernández-Tangpi \cite{hernandez2023propagation}, in which 
the authors use a probabilistic approach to recast a mean-field Schr\"odinger bridge into a stochastic optimization problem with McKean-Vlasov dynamics, and connect the optimal control to a solution to a forward backward SDE of McKean-Vlasov type (MKV FBSDE). However, given the generality of the drift, diffusion, and running cost functions, the associated MKV FBSDE is derived without a discussion of uniqueness.  In fact, it is not completely obvious that a McKean-Vlasov-type SBP can be converted to a McKean-Vlasov stochastic control problem via the usual Girsanov theorem argument, as we show in Remark \ref{rem:1} below. Moreover, the conditions imposed on the penalty function 
$G$ are abstract and can be difficult to verify in common examples. By contrast, our framework leverages the PDE formulation and Doob’s $h$-transform representation, requiring only mild growth conditions on 
$G$ and control of density gaps (see Assumption~\eqref{assump:general penalty term}). Several concrete examples of admissible 
$G$ are provided in Example \ref{G1} and \ref{G2}.

\ms

The rest of the paper is organized as follows. Section \ref{sec:preliminary} introduces the necessary concepts and notations. In particular, we present the connection between the underlying SBP and its stochastic control formulation, and introduce the notion of the SCSBP together with some potential applications. Section \ref{sec:existence_opt_policy} is devoted to the existence of optimal policies for the SCSBP at each penalty parameter $k$, while Section \ref{sec:conv_opt_policy} establishes that the penalized optimal policies converge to those of the original SBP as $k \to \infty$, with a linear rate of convergence. In addition, we prove convergence of the corresponding value functions and quantify the distance between the terminal distribution and the target distribution in terms of the Wasserstein distance. Sections \ref{Stability} and \ref{sec:general} are devoted to the case with a general initial distribution. A crucial stability result is established in Section \ref{Stability}, via the stability of Schr\"odinger potentials of an auxiliary entropy optimal transport problem, and in Section \ref{sec:general} we complete the fixed-point argument.

\section{Preliminary}\label{sec:preliminary}

Throughout this paper,  we consider  a  generic Euclidean space $\hX$, and regardless of its dimension,  we denote $(\cd,\cd)$ and $|\cd|$ be its inner product and norm, respectively. We denote $\hC([0,T];\hX)$ to be the space of $\hX$-valued continuous functions defined on $[0,T]$ with the usual sup-norm, and in particular, we denote $\hC^d_T:=\hC([0,T];\hR^d)$, and  let $\sB(\hC_T^d)$ be its topological Borel field.
We shall consider the following 
{\it canonical probabilistic space}: $(\O,\cF, \hP)$, where $(\O, \cF):=(\hC^d_T, \sB(\hC^d_T))$ and $\hP\in\sP(\O)$, the space of all probability measures defined on $(\O, \cF)$. Finally, we let $\hP^0\in \sP(\O)$ be
the Wiener measure on $(\O, \cF)$;  $W_t(\o):=\o(t)$, $\o\in\O$, the canonical process; and  $\hF^0:=\{\cF^0_t\}_{t\in[0,T]}$,  where $\cF^0_t:=\sB_t(\hC_T^d):=\si\{\o(\cd\wedge t):\o\in\hC^d_T\}$, $t\in[0,T]$. We assume that $\hF^0$ has the usual $\hP^0$-augmentation so that it satisfies the {\it usual hypotheses} (cf. e.g., \cite{protter2005stochastic}), and for $p\ge 1$, we denote $\hL^p_{\hF^0}([0,T];\hX)$ to be all $\hX$-valued, $p$-integrable, $\hF^0$-adapted processes. Finally, we denote $\sM(\hR^d)$ to be all $\si$-finite measures on $\hR^d$ and $\sP_p(\hR^d)$ to be all probability measures with finite $p$-th moment on $\hR^d$ equipped with $p$-Wasserstein metric, denoted by $W_p(\cd, \cd)$. 	

\ms

We recall that a classical Schr\"{o}dinger bridge problem (SBP) amounts to finding, for  
$\hP\in\sP(\O)$, 
\bea
\label{SP0}
V(\hP):= \inf_{\hQ\in\sP(\mi, \mt)} D_{\rm KL}(\hQ\|\hP),
\eea
where 
$D_{\rm KL}(\,\cd\,\|\,\cd\,)$ is the {\it Kullback-Leibler Divergence} or the {\it Relative Entropy} (cf. \cite{KL}) \footnote{We remark here that the KL divergence $D_{\rm KL}(\,\cd\,\|\,\cd\,)$ can be easily extended to any $\si$-finite measures. In this case $D_{\rm KL}(\hQ\|\hP):=\int  \log\big(\frac{ \hQ(dx)}{\hP(dx)}\big)\hQ(dx)$, if $d\hQ \ll d \hP$.}, defined by
\beaa
D_{\rm KL}(\hQ\|\hP)
:=
\begin{cases}
	\hE^{\hQ}\big[\log\big(\frac{d\hQ}{d\hP}\big)\big], & {\text{if }} d\hQ \ll d \hP;\\
	\infty, &{\text{otherwise}}.
\end{cases}
\eeaa
We note that the admissible set $\sP(\mi, \mt)\subseteq \sP(\O)$ in (\ref{SP0}) is defined as the collection of all probability measures $\hQ\in \sP(\O)$, such that $\hQ$ is the law of some continuous $\hR^d$-valued continuous process $X$ defined on the canonical space with probability measure ${\bar\hP}\in\sP(\O)$ (i.e., $\hQ= \bar\hP\, \circ\, X^{-1}$) satisfying ${\hQ}\circ{X_0}^{-1}=\mi $ and ${\hQ}\circ{X_T}^{-1}\neg =\neg \mt\neg$. In other words, $\sP(\mi, \mt)$ contains all   ``path measures'' with initial and terminal distributions being $\mi$ and $\mt$, respectively.

We note that if $\hP=\hP^0$ is Wiener measure  and   $\hQ\in\sP(\O)$ is absolutely continuous with respect to $\hP^0$, then by the Cameron-Martin-Girsanov theorem, there exists $\a\in\hL^2_{\hF^0}([0,T];\hR^{d})$, such that
$$
\frac{d\hQ}{d\hP^0}\Big|_{\cF_t}:=\sE_t(\a)=\exp\Big\{\int_0^t \a_s dW_s-\frac12\int_0^t |\a_s|^2ds\Big\}, \q t\in[0,T],
$$
is a $\hP^0$-martingale, and $\widetilde{W}_t:=W_t-\int_0^t \a_sds$, $t\in[0,T]$, is a $\hQ$-Brownian motion. 
In this case, one can then easily check  that
\bea
\label{DKL}
D_{\rm KL}(\hQ\|\hP^0)=\frac12 \hE^\hQ\Big[\int_0^T|\a_t|^2dt\Big].
\eea

\paragraph{Schr\"odinger Bridge and  Related Control Problem.}

In light of (\ref{DKL}), one can easily  associate a SBP to a stochastic control problem (see, e.g., \cite[\S 4.4]{chen2021stochastic},  \cite[\S 1]{chow2022dynamical} and \cite[p36]{leonard2013survey}). Consider a standard SDE 
on a canonical space  $(\O, \cF, \hP)$, with $\hP\in\sP(\O)$:
\bea
\label{SBSDE}
d X_t =  b(t,X_t)d t + d W_t, \quad \hP\circ (X_0)^{-1} = \mi.
\eea
In what follows we shall denote, for $t\in[0,T]$, $\hP_{X_t}:=\hP\circ (X_t^{-1})\in\sP_2(\hR^d)$.

Then we  consider the following SBP:
\begin{equation}
	\label{SBP1}
	V(\mi,\mt) :=
	\inf_{ \hQ \in \sP(\mi, \mt)}D_{\rm KL}(\hQ\|\hP).
\end{equation}
Similar to (\ref{DKL}), we now recast (\ref{SBP1}) as the following stochastic control problem:
\bea
\label{SB-objective}
V(\mi,\mt)=\inf_{\alpha \in \mathcal{A}}J(\alpha) = \inf_{\alpha \in \mathcal{A}}\hE^\hQ \Big[ \frac{1}{2} \int_0^T |\alpha_s|^2ds  \Big],
\eea
where $\hQ\in\sP(\O)$ is such that $\frac{d\hQ}{d\hP}=\sE(\a)$ for some $\a\in \mathcal{A}\subseteq \hL^2_{\hF^0}([0,T];\hR^d)$, under which the underlying controlled dynamics takes the form:
\bea
\label{SB-dynamics}
d X^\alpha_t = [b(t,X^\alpha
_t) +\alpha_t]d t + d \wt{W}_t, \quad \hQ_{X_0^\alpha} =\mi, \q \hQ_{X_T^\alpha}= \mt,
\eea
where $\wt{W}$ is a $\hQ$-Brownian motion. 

\begin{Remark}[Subtlety in formulating the McKean-Vlasov version of the problem]
	\label{rem:1}
	{\rm	It is rather tempting to apply the idea above to the so-called McKean-Vlasov SBP (MVSBP). Consider, for example, 
		the following McKean-Vlasov SDE on $(\O, \cF, \hP)$:
		\beaa
		d X_t = b (t,X_t,\hP_{X_t}) d t +   d W_t, \quad  \hP_{X_0}= \mi.
		\eeaa
		Similar to (\ref{SBP1}) we can define an SBP, which we shall 
		refer to  as an MVSBP.  Again, by (\ref{DKL}), we  recast such MVSBP as the following (weak form)   stochastic control problem:
		\bea
		\label{MFC}
		V(\mi,\mt) :=  \inf_{\alpha} \hE^\hQ\Big[\int_0^T\frac{1}{2}|\alpha_t|^2\dd t \Big].
	\end{eqnarray}
	Then, $\hQ\in\sP(\O)$ must be such that the underlying controlled process $X^\a$ becomes, under $\hQ$:
	\bea
	\label{MFC-dynamics}
	d X^\a_t = [\alpha_t + b (t,X^\a_t,{\hP_{X^\a_t}})]  d t + d \widetilde{W}_t, \quad  \hQ_{X^\a_0} =\mi, \q \hQ_{X^\a_T}= \mt,
\end{eqnarray}
where $\wt{W}$ is a $\hQ$-Brownian motion, and   $\a\in\hL^2_{\hF^0}([0,T];\hR^d)$. However, by looking at (\ref{MFC-dynamics}) more closely we see that   the problem (\ref{MFC}) and (\ref{MFC-dynamics}) {\it do not} form a McKean-Vlasov control problem, since $\hP_{X_t} \neq \hQ_{X_t}$(!). Therefore, an MVSBP should be formulated more carefully so as to connect to an McKean-Vlasov stochastic control problem.
\qed}
\end{Remark}

Ideally, the optimal solution to the Schrödinger bridge problem \eqref{SB-objective}-\eqref{SB-dynamics} provides a transport map from the initial distribution $\mi$ to the target  distribution $\mt$. This transport is interpolated by a diffusion process that most closely resembles the canonical Brownian motion in the space of path measures. However, designing training algorithms to (approximately) learn the optimal solution to \eqref{SB-objective}–\eqref{SB-dynamics} typically involves an iterative scheme that alternately relaxes the initial and terminal constraints \cite{de2021diffusion,vargas2021solving,wang2021deep,shi2024diffusion}, and whose convergence rate and computational complexity in high-dimensional settings remain unclear.

While there could be techniques in stochastic control theory to deal with such a constraint, we shall follow the idea of \cite{hernandez2023propagation}, by approximating the original control problem by a family of   {\it unconstrained}
McKean-Vlasov stochastic control problems with terminal penalties. More precisely, we shall allow the law of $X^\alpha_T$ in \eqref{SB-dynamics} to be different from $\mt$, but add a corresponding penalty function to it in the cost functional $J(\cd)$. 

To this end, let us still denote $\cA\subseteq \sP(\O)$ to be the set of all $\hQ\in\sP(\O)$ such that 

(i) $\frac{d\hQ}{d\hP}\Big|_{\cF_T}=\sE(\a)=\exp\big\{\int_0^T \a_sdW_s-\frac12\int_0^T|\a_s|^2ds\big\}$, $\a\in\hL^2_{\hF^0}([0,T];\hR^{d})$;

(ii) Under $\hQ$, the underlying state process $X$ follows the dynamics:
\bea
\label{RSB_dynamics}
d X^\alpha_t = [b(t,X^\alpha_t) + \sigma(t) \alpha_t]d t + \sigma(t)d\wt{W}_t, \quad \hQ_{X_0^\alpha} = \mi,	
\eea
where $\wt{W}$ is a $\hQ$-Brownian motion. 

Throughout this paper, we shall make the following {\it Standing Assumption} on the coefficients $b$ and $\si$.
\begin{Assumption}
\label{assump1}
The coefficients $b:[0,T]\times\hR^d\to \hR^d$ and $\si:[0,T]\to \hR^{d\times d}$ are given deterministic  continuous functions, such that there exists $L>0$, it holds that
$$ |b(t,x)-b(t,y)|\le L|x-y|, \qq t\in[0,T], ~x,y\in\hR^d. 
$$
\end{Assumption}
Clearly, under Assumption \ref{assump1}, the SDE (\ref{RSB_dynamics}) has a unique strong solution $X^\a$ on $(\O, \cF, \hP^0)$ for any given $\a\in \hL^1_{\hF^0}([0,T];\hR^{d\times d})$  (see \cite{protter2005stochastic,zhang2017backward}). We shall often identify $\hQ\in \cA$ with its associated process $\a$, and denote $\hQ\sim \a$ and $\a\in\cA$ when the context is clear. The key element of the soft-constrained Schr\"odinger bridge problem is the following penalty function.

\begin{Definition}
\label{penalty}
A function $G(\cd)=G(\cdot;\mt):\sP_2(\hR^d) \to [0,\infty)$ is called a continuous penalty function associated to $\mt\in \sP_2(\hR^d)$ if:  (1) $G(\cd)$ is continuous on $\sP_2(\hR^d)$; and (2) 
$G(\mu;\mt) = 0$ if and only if $\mu=\mt$.
\qed
\end{Definition}

Now let us introduce the following family of  McKean-Vlasov-type stochastic control problems:

\begin{Problem}[Soft-constrained Schr\"odinger bridge problem (SCSBP)] 
\label{softSBP}
For $k\in\hN$, find $\h\a^k\in\cA$ such that $V^k: =J^k(\widehat\a^k)=\inf_{\alpha \in \cA} J^k(\alpha) $, where
\bea
\label{RSB_objective}
J^k(\alpha)=\hE^\hQ \left[ \frac{1}{2} \int_0^T |\alpha_s|^2 d s  + k G({\hQ_{X^\alpha_T}})\right],
\eea
and   $G(\cd)=G(\cdot;\mt)$ is the given penalty  function satisfying Definition \ref{penalty}  and $\hQ\sim \a$.
\qed
\end{Problem}

\paragraph{Applications in Generative AI.} We remark that the SCSBP Problem \ref{softSBP}
offers a general framework that can be applied to address multiple problems in generative AI. We briefly mention a few motivational examples.

\begin{Example}[Data generation] {\rm The goal of generative AI is to train a data generation procedure using a finite number of iid. data samples $\{x_1,\cdots,x_N\}$ under a (unknown) target distribution $\mt$, in order to simulate  unlimited number  of data samples whose underlying distribution is close to $\mt$ \cite{song2020score,Ho20,han2024neural}.

To cast this problem into our framework, we can take, for example, $\mi = \mathcal{N}(0,I)$ and $\mt=p_{\rm data}$ in the theoretical framework (or $\mt =\frac{1}{N}\sum_{i=1}^N \d_{ x_i}$ in the practical implementation). Then the optimal control $\widehat{\alpha}$ of SCSBP leads to a controlled process $(X_t^{\widehat \alpha})_{0\leq t\leq T}$ that simulates the data output $X_T^{\widehat \alpha}$. Our key results (see Theorem \ref{thm:convergence_opt_policy} and Theorem \ref{thm:convergence_rate_general_initial} below) show that  the terminal measure $\mathbb{Q}_{X_T^{\widehat \alpha}}$ is close to $\mt$,  when $k$ is sufficiently large.
\qed}
\end{Example}

\begin{Example}[Fine-tuning under a reward signal]
\label{exmpl:fine-tuning}
{\rm Fine-tuning a diffusion model means taking a pre-trained model and training it further on a smaller, task-specific dataset so it learns to generate outputs more suited to that new context \cite{tang2024fine,zhao2025fine,uehara2024fine,han2024stochastic}. For example, a diffusion model trained on general images can be fine-tuned to generate a specific  style (evaluated via a reward function). This process updates the model’s parameters just enough to adapt to the new data, without starting training from scratch.

In terms of our framework, we can consider \reff{SBSDE} as a pre-trained model  with the drift $b(t,x):=s_{\widehat{\theta}}(t,x)$ being a well-trained score function, and $\widehat{\theta}$ is the trained parameter. Note that, as the  result of pre-training, the output measure $\mathbb{Q}_{X_T}$ is sufficiently close to some  data distribution $\mt$. We then introduce a fine-tuning procedure through a reference measure $p_{\rm ref}$ with density $\frac{\exp(r(x))}{\int_{\mathbb{R}^d} \exp(r(x)) dx}$, where $r: \mathbb{R}^d \rightarrow \mathbb{R}$ is a given reward function satisfying $\int_{\mathbb{R}^d} \exp(r(x)) dx<\infty$. Now replacing $\mt=p_{\rm ref}$,   the optimal control  $\widehat\alpha$ of the corresponding SCSBP can then serve as the fine-tuning score function; and consequently, the new drift term $b(t,X^{\widehat\alpha}_t) +\sigma(t)\widehat\alpha$  acts as a combined score function. 

Clearly, in this application the penalty parameter $k$ should not be chosen too large; otherwise, the effect of the preference function may dominate the fidelity to the original data distribution. With an appropriately selected $k$, the resulting measure $\mathbb{Q}_{X^{\widehat\alpha}_T}$ not only reflects $p_{\rm data}$ but also integrates the reward function $r$. In contrast,   we remark that the classic SBP \eqref{SP0} is not capable of handling this application as it has a {\it pre-fixed} target distribution.
\qed}	
\end{Example}

\begin{Example}[Transfer learning] 
\label{exmpl:transfer-learning}
{\rm Transfer learning is a machine learning approach where knowledge gained from a ``source task'' is reused to improve learning in a related but different ``target task'' \cite{cao2023risk,torrey2010transfer,ouyang2024transfer}. In what follows we shall consider transfer learning in the context of data generation.

Let us consider a source task \((Y_{\rm sou})\), characterized by a distribution   \(p_{\rm sou}\), and a target task \((Y_{\rm tar})\) with distribution \(p_{\rm tar}\). Typically, \(p_{\rm sou}\) and \(p_{\rm tar}\) are assumed to be close under a certain divergence or distance function $G(p_{\rm tar};p_{\rm sou})$ (assuming \( G \geq 0 \)), such as the Wasserstein distance \cite{ouyang2024transfer}.

To fit the transfer learning into our framework, we can take \( \mi = p_{\rm sou} \), \( \mt = p_{\rm tar} \), and set  $b\equiv 0$ for simplicity. In this case, if we choose $\a=0$, and $X_0\sim p_{\rm sou}$, then $X_T=X_0+W_T\sim p_{\rm sou} * \mathcal{N}(0,T\hI_d)$, where $ \mathcal{N}(0,T {\hI_d})=\hP^0_{ W_T}$ and $\hI_d$ denotes the $d\times d$ identity matrix. Thus, denoting  the optimal control by $\widehat{\alpha}$ and noting that $\alpha\equiv0$ is sub-optimal, we must have
\begin{eqnarray*}
	\hE^\hQ \left[ \frac{1}{2} \int_0^T |\widehat\alpha_s|^2 d s  \right] \leq  \hE^\hQ \left[ \frac{1}{2} \int_0^T |\widehat\alpha_s|^2 d s  + k G(\hQ_{X^{\widehat{\alpha}}_T})\right] \leq k G(p_{\rm sou} * \mathcal{N}(0,T {\hI_d});p_{\rm tar}).
\end{eqnarray*}
This implies that the optimal control \( \widehat{\alpha} \) has a small \( L^2 \)-norm, indicating only minor adjustments are required during sampling—provided \( k \) is not too large.
\qed}
\end{Example}

{\color{black}

\section{Existence of  Optimal Policies for SCSBP's}
\label{sec:existence_opt_policy}

In this section we study the stochastic control problem (\ref{RSB_dynamics})-(\ref{RSB_objective}) and the associated soft-constrained SBP. In particular, we shall prove that the optimal control for each $k\in\hN$ exists and in next section we will show that these optimal policies will converge  to the solution of the original SBP, with a linear rate of convergence. We shall assume that the target distribution for the SCSBP  has
density  $f_{\rm tar} \in \hL^1(\hR^d)$.  Also, we assume 
$\sigma(\cdot)=I_d$, that is,   we consider the following SDE:
\bea
\label{SDE0}
d X_t = b(t,X_t)d t  + d W_t, \q {\hP}_{X_0} = \mi.
\eea
In what follows we shall often consider the SDE with initial data $(t,x)\in[0, T]\times\hR^d$:
\bea
\label{SDEtx}
X_s=x + \int_t^s b(r,X_r)dr + W_s-W_t,\q s\in[t,T].
\eea
We denote the solution of (\ref{SDEtx}) by $X^{t,x}$. Clearly, if we let $p(\cd, \cd;\cd, \cd)$ be the transition density of the solution $X$ to (\ref{SDE0}), that is,  $\hP\{X_s\in dz|X_t=x\}=p(s,z; t, x)dz$,  $0\leq t<s\leq T$, $z,x\in \mathbb{R}^d$, then $p(s,\cdot; t, x)$ is the density of $X^{t,x}_s$. It is well known that $p(\cd,\cd;\cd,\cd)$ is the fundamental solution to Kolmogorov backward (parabolic) PDE, and under mild conditions (see, e.g., \cite{aronson1967bounds}), there exist $c_1$, $c_2$, $\lambda$, $\Lambda>0$, it holds that  
\begin{equation}
\label{aronson}
c_1 (s-t)^{-\frac{d}2}e^{-\frac{\lambda |z-x|^2}{s-t}}<p(s,z;t,x)< c_2(s-t)^{-\frac{d}2}e^{-\frac{\Lambda |z-x|^2}{4(s-t)}}.
\end{equation}

Keeping the original SBP  (\ref{SB-objective}) associated with (\ref{SDE0}) in mind, let us now recall the  Problem \ref{softSBP} and the cost functional $J^k(\alpha)$  defined by (\ref{RSB_objective}).  
For notational simplicity, we let $\hP_{X^\a}$ (resp. $\hP_X$) denote the probability measure on $\hC^d_T$ induced by $X^\a$ (resp. $X$). Clearly, when $\alpha\equiv 0$, we have $\hP_{X^0}=\hP_{X}$, where $X$ solves (\ref{SDE0}). Furthermore, we shall denote $\hE[\cd]=\hE^{\hP}[\cd]$ when context  is clear, and for each $k\in\hN$ we can easily check that
\bea
\label{RSB-obj1}
J^k(\alpha)=\hE\Big[ \frac{1}{2} \int_0^T |\alpha_s|^2 d s  + k G(\hP_{X^\alpha_T})\Big]=D_{\rm KL}(\mathbb{P}_{X^\alpha}\|\mathbb{P}_X) + k G(\hP_{X^\alpha_T}). 
\eea
Now let us define, for each $k\in\hN$,  a mapping $D_k(\cdot): \sP_2(\hR^d) \to \hR$ by 
$$D_k(\mu) = D_{\rm KL}(\mu\|\hP_{X_T}) + kG(\mu),$$
and note that $D_{\rm KL}(\mathbb{P}_{X^\alpha}\|\mathbb{P}_X) \geq D_{\rm KL}(\mathbb{P}_{X^\alpha_T}\|\mathbb{P}_{X_T})$,
we deduce from (\ref{RSB-obj1}) that
\bea
\label{Ineq_Jk}
J^k(\alpha) \geq          D_{\rm KL}(\mathbb{P}_{X^\alpha_T}\|\mathbb{P}_{X_T}) + kG(\hP_{X^\alpha_T})=D_k(\hP_{X^\a_T}).
\eea     
If $\widehat \a$ is the optimal control corresponding to the original SBP, that is, $\hP_{X^{\widehat\a}_T}=\mt$, then by definition of the penalty function $G(\cd)$, we should have $G(\mt)=0$, and therefore,
\beaa
D_k(\hP_{X^{\widehat\a}_T})=D_k(\mt)= D_{\rm KL}(\mt\|\mathbb{P}_{X_T}) + kG(\mt)=    D_{\rm KL}(\mt\|\mathbb{P}_{X_T}).
\eeaa

Throughout the rest of this section, we shall focus on the special case: $\mi=\d_{x_0}$ for some $x_0\in\hR^d$. The case with general initial distribution $\mi$ will be studied in Sections \ref{Stability} and \ref{sec:general}. We begin with the following well-known 
result from \cite{dai1991stochastic}, which will play an important role in our discussion.   
\begin{Lemma}[ {\cite[Theorem 3.1]{dai1991stochastic}}]
\label{lemma:sdb}    
Let  $X$ be a weak solution to \eqref{SDE0} with $X_0=x_0\in \mathbb{R}^d$ (i.e., $\mi=\delta_{x_0}$). Assume that $D_{\rm KL}(\mt\|\hP_{X_T }) < \infty$. Then, the optimal solution to the SBP \eqref{SBP1}-\eqref{SB-objective} is given by $\widehat\alpha_t = \nabla \log h(t, X^{\widehat\alpha}_t)$, where 
\begin{eqnarray}
	\label{eq:definitionofh}
	h(t,x) = \int_{\hR^d} p(T,z;t,x)\frac{f_{\rm tar}(z)}{p(T,z;0,x_0)}d z =  \mathbb{E}\Big[\frac{f_{\rm tar}(X_T)}{p(T, X_T;0,x_0)}\Big|X_t=x \Big],
\end{eqnarray}
for $(t,x)\in[0,T]\times\hR^d$.
\qed
\end{Lemma}
Next, we make the following assumptions on the penalty function $G$:
\begin{Assumption}
\label{assump:general penalty term}
(i) There exists some small constant $\varepsilon>0$ such that
\begin{equation}
	\label{assum:G1}
	G(\mu)\rightarrow +\infty,\quad \mathrm{as} \,\,\|\mu\|^{2+\varepsilon}\rightarrow +\infty.
\end{equation}
where $\|\mu\|^{p} := \int_{\hR^d} |x|^p \mu(d x)$ for any $p>0$. 

\ms
(ii) There exist   $C, \l >0$, and a function $\phi:\hR^d\to (0,1]$ satisfying $\phi(x)e^{\lambda|x-x_0|^2}\le C$,  such that for any $\mu\in \sP_2(\mathbb{R}^d)$  with density function $f_\m$, it holds that
\begin{eqnarray}\label{condition: dominance}
	|f_\m(x) - f_{\rm tar}(x)| \leq C \phi(x)G(\mu), \quad   x\in \mathbb{R}^d.
\end{eqnarray}
\end{Assumption}
\begin{Remark}
\label{rem:3.3}
{\rm (i) The function $G(\m)$ on the right hand side of   \eqref{condition: dominance} should read  $|G(\mu)-G(\mt)|$, as $G(\mt)=0$, which essentially states that if $\m$ is close to $\mt$ in terms of $G$, then $f_\m$ is close to $f_{\rm tar}$ in $\hL^1$. 
	
	(ii) A slightly stronger consequence of \eqref{condition: dominance} is the following. Recall the (generalized) Kantorovich and Rubinstein dual representation (cf. e.g., \cite{Edward}): denoting Lip$(1)$ to be all Lipschitz functions $\varphi:\hR^d\to \hR$ with Lipschitz constant  Lip$_\varphi\leq 1$ (hence $|\f(x)|\le C(1+|x|)$),  then it holds that
	\beaa
	W_{1}(\mu,\mt )=\sup_{\f\in Lip(1)}\left\{\int _{\hR^d}\varphi(x)(f_\m(x) - f_{\rm tar}(x))dx\right\}\leq 
	CG(\mu) \int _{\hR^d}(1+| x|) \phi(x)dx. 
	\eeaa
	This suggests  that   $G(\mu)\sim 0$ implies that $\mu$ is close to $\mt$ in the sense of $W_1$.
	\qed}
\end{Remark}

Before we proceed, let us give two examples that justify Assumption \ref{assump:general penalty term}. 
\begin{Example} \label{G1}
{\rm We consider the class of $\mu$ and $\mt$ such that
	\begin{equation*}
		G(\m):= \int_\hR |x|^p |f_\m (x) - f_{\rm tar}(x)|dx,
	\end{equation*}
	is well-defined for a given $p$.
	Clearly, Definition \ref{penalty}-(i) and Assumption \ref{assump:general penalty term}-(i) are satisfied when $p>2$. Assumption \ref{assump:general penalty term}-(i) holds when 
	\begin{eqnarray*}
		\phi(x) \leq \frac{|f_\m (x) - f_{\rm tar}(x)|}{c \int_\hR |x|^p |f_\m (x) - f_{\rm tar}(x)|dx},
	\end{eqnarray*}
	for all $\mu$ in the collection one may consider.
	\qed
}
\end{Example}

Another natural example of $G$ satisfying Assumption \ref{assump:general penalty term} would be the Wasserstein distance or the KL divergence, augmented with a small ``guardrail'' term that enforces the uniform (weighted) pointwise control in \eqref{condition: dominance}. This guardrail can be taken as a weighted $\hL_\infty$ norm, a H\"older \(\mathbb{C}^\alpha\) seminorm, or an RKHS norm (e.g., with kernel \(k(x,y)=\phi(x)\phi(y)\kappa(x-y)\)). In next example we illustrate such a choice with the Wasserstein distance plus an $\hL_\infty$ guardrail.
\begin{Example} \label{G2}
{\rm 
	Consider the case that $\mt\in \sP_{p}(\mathbb{R}^d)$ with $p>2$.
	We define, for $c>0$ and $\phi(x) = \exp(-\lambda|x-x_0|^2)$ with some $x_0\in \mathbb{R}^d$,
	\begin{eqnarray*}
		G(\mu):= W_2(\mu,\mt) + c \Big\|\frac{f_{\mu}-f_{\rm tar}}{\phi}\Big\|_{\hL^\infty}.
	\end{eqnarray*}
	Then, it is easy to check that
	\begin{eqnarray*}
		|f_{\mu}(x)-f_{\rm tar}(x)|\leq \Big\|\frac{f_{\mu}-f_{\rm tar}}{\phi}\Big\|_{\hL^\infty} \phi(x)\leq \frac{1}{c}\phi(x)G(\mu).
	\end{eqnarray*}
	Thus \eqref{condition: dominance} holds  and $\phi(x)e^{\lambda|x-x_0|^2}\leq C$ holds with $C=\max\{1,\frac{1}{c}\}$.
	
	Let $\{\mu_n\}_{n \ge 1}\subset \sP_2(\mathbb{R}^d)$ with $\|\mu_n\|^{2+\varepsilon} \rightarrow \infty$.  We claim that $G(\mu_n)$ must be unbounded.  Indeed, suppose not. Then  there exists $M,M'>0$ such that
	$W_2^2(\mu_n,\mt)\leq M$ and $\|\phi^{-1}(f_{\mu_n}-f_{\rm tar})\|_{\hL^\infty}<M'$ for all $n\in \mathbb{N}_+$. Hence $f_{\mu_n}(x)\leq f_{\rm tar}(x)+M'\phi(x)$, $x\in\hR^d$.
	Integrating against $|x|^{2+\varepsilon}$ and using the facts that $\mt\in \sP_{2+\varepsilon}(\mathbb{R}^d)$ with $\varepsilon=p-2>0$ and $\int|x|^{2+\varepsilon}\phi(x)dx<\infty$, we have
	\begin{eqnarray*}
		\|\mu_n\|^{2+\varepsilon}\leq \|\mt\|^{2+\varepsilon} + M'
		\int |x|^{2+\varepsilon}\phi(x)dx<\infty, \qq n\in\hN.
	\end{eqnarray*}
	This contradicts the fact that $\|\mu_n\|^{2+\varepsilon}\rightarrow \infty$, proving the claim. Hence
	\eqref{assum:G1} holds.
	\qed}
\end{Example}

We are now ready to investigate the existence of optimal control of Problem \ref{softSBP}  for each $k\in\hN$, which would be essential for our approximation scheme.  Recall that in the rest of the section we assume that $\mi=\d_{x_0}$ for some $x_0\in\hR^d$.
To begin with, we first claim that for each $k\in\hN$, there exists  $\widehat\mu_k \in \sP_2(\hR^d)$ such that the static optimization problem on the measure space has a solution
\bea
\label{infDk}
D_k(\widehat \m_k)=\inf_{\m\in\sP_2(\hR^d)} D_k(\mu).
\eea
Indeed, let $X$ be the solution to uncontrolled SDE (\ref{SDE0}), and $\mt$ be given such that $D_{\rm KL}(\mt\| \hP_{X_T})<\infty$. Since $ G(\mt)=0$, we have  
\beaa
\label{m}
m := D_k(\mt) = D_{\rm KL}(\mt\|\hP_{X_T}) + kG(\mt) = D_{\rm KL}(\mt\|\hP_{X_T})<\infty.
\eeaa
Next, let us define, for fixed $k\in\hN$, a set
\begin{eqnarray*}
\mathcal{S}_k:=   \Big\{\mu\in \sP_2(\mathbb{R}^d): D_k(\mu) \leq m\Big\}.
\end{eqnarray*}
Clearly, $\mathcal{S}_k\neq\emptyset$ since $\mt\in \mathcal{S}_k$, and by \eqref{assum:G1}, there exists $M_k>0$ such that
$\|\mu\|^{2+\e} \leq M_k$, for all $\mu\in \mathcal{S}_k$. Thus $\cS_k$ is uniformly integrable in $\hL^2$. Now let $\{\mu_{k}^{(i)}\}_{i=1}^{\infty}\subset\sP_2(\hR^d)$ be a minimizing sequence, namely, 
\begin{eqnarray*}
\lim_{i\rightarrow\infty} D_k(\mu_{k}^{(i)})= \inf_{\mu\in \mathscr{P}_2(\mathbb{R}^d)} D_k(\mu).
\end{eqnarray*}
Since $ \inf_{\mu\in \mathscr{P}_2(\mathbb{R}^d)} D_k(\mu) \leq D_k(\mt) =  m$, we may assume without loss of generality that $\{\mu_{k}^{(i)}\}_{i=1}^{\infty}\subset \mathcal{S}_k$. Since $\mathcal{S}_k$ is uniformly integrable and is tight
in $\sP_2(\mathbb{R}^d)$, there exists  subsequence $\{\mu_{k}^{(i_l)}\}_{l=1}^\infty$ such that 
$\widehat\mu_k:=\lim_{l\rightarrow\infty}\mu_{k}^{(i_l)}\in \sP_2(\mathbb{R}^d)$. \footnote{ This follows from the result on Wasserstein distance vs. weak convergence (see, e.g., \cite[Theorem 7.12]{Villani}), which states that $W_p(\m_k, \m)\to 0$ if and only if $\m_k\Rightarrow \m$ weakly, and $\lim_{R\to\infty}\limsup_{k\to\infty}\int_{\{|x-x_0|\ge R\}}|x-x_0|^p\m_k(dx)=0$.}.
Since the mapping $\m\mapsto D_k(\m)$ is continuous, we have
\begin{eqnarray*}
D_k(\widehat\mu_k) 
= D_{\rm KL}(\widehat\mu_k\|\mathbb{P}_{X_T}) + kG(\widehat\mu_k) 
=  \inf_{\mu\in \sP_2(\mathbb{R}^d)}  D_k(\mu),
\end{eqnarray*}
proving the claim. Furthermore, if we denote the density of $\hP_{X_T}$ by $f_{X_T}$, and note that $D_{\rm KL}(\widehat\mu_k\|\hP_{X_T}) \leq D_k(\widehat\mu_k) \leq m<\infty$, we know that $\frac{d\widehat\mu_k}{d\hP_{X_T}}$ exists and $$\frac{d\widehat\mu_k }{dx}(x)=\frac{d\widehat\mu_k}{d\hP_{X_T}} \cd f_{X_T}(x)  =: f_k(x).$$

Keeping the above discussion in mind, we are now ready to prove the following theorem.
\begin{Proposition}
\label{existak}
Assume that  Assumption \ref{assump:general penalty term} is in force, and that $\mi = \delta_{x_0}$, $x_0\in \mathbb{R}^d$. 
Then, for each $k\in\hN$, the optimal control for Problem \ref{softSBP}, denoted by $\widehat{\a}^k$, exists. Furthermore,  $\widehat{\a}^k$ has the following explicit feedback form: $\widehat\alpha^k_t := \nabla \log h^k(t, X^{\widehat\alpha^k}_t)$, where 
\begin{align}
	\label{h*k}
	h^k(t,x) 
	= \int_{\hR^d} \frac{f_k(z)}{p(T,z;0,x_0)} p(T,z;t,x) d z
	= \hE\Big[ \frac{f_k(X_T)}{p(T, X_T;0,x_0)}\Big|X_t=x\Big].
\end{align}
\end{Proposition}

\begin{proof} Let $k\in\hN$ be fixed, and let $\widehat{\m}_k$ be the minimizer of $D_k(\cd)$ defined by (\ref{infDk}). Then, by \eqref{Ineq_Jk}, for
any $\alpha\in \hL^2_{\hF^0}([0,T];\hR^d)$, we have 
\begin{eqnarray*}
	J^k(\alpha) \geq D_k(\hP_{X^\alpha_T})\geq D_k(\widehat\mu_k), 
\end{eqnarray*}
Therefore, in order to find the optimal control for Problem \ref{softSBP} , it suffices to find   $\widehat\alpha^k$ such  that (i) $\ X^{\widehat\alpha^k}_T \sim  \widehat\mu_k$; and (ii) $J^k(\widehat\alpha^k) =D_k(\hP_{X^{\widehat\alpha^k}_T})$.

To this end, we first recall that $\widehat\mu_k$ is the minimizer of the function $D_k(\cdot)$ with density $f_k$. Next, we apply Lemma \ref{lemma:sdb} with $\mt$ being replaced by $\widehat{\mu}_k$ to get the optimal control $\widehat\alpha^k\in\sP(\mi, \widehat{\mu}_k)$ for the original SBP (\ref{SB-objective})-(\ref{SB-dynamics}), which satisfies  $\widehat\alpha^k_t=\nabla \log h^k(t,X^{\widehat\alpha^k}_t)$, where $h^k$ is defined by (\ref{h*k}), and $\ X^{\widehat\alpha^k}_T\sim \widehat\mu_k$. Now, note that for this SBP we have 
$$V(\mi, \widehat{\mu}_k)=\frac{1}{2}\hE\Big[\int_0^T |\widehat\alpha^k_s|^2 d s\Big] = D_{\rm KL}(\widehat\mu_k\|\hP_{X_T}), 
$$
we conclude that  
\begin{equation*}
	J^k(\widehat\alpha^k) = D_{\rm KL}(\widehat\mu_k\|\hP_{X_T}) + k G(\widehat\mu_k) = D_k(\widehat\mu_k)=D_k(\hP_{X^{\widehat\alpha^k}_T}).
\end{equation*}
In other words, $\widehat\alpha^k$ is indeed the optimal control for the Problem \ref{softSBP}, proving the proposition.
\end{proof}   

\section{Convergence Results under Delta Initial Distribution}
To show the convergence of both policies and value functions, we first make the following observations. First, if we denote $g(x):=\frac{f_{\rm tar}(x)}{p(T, x;0,x_0)}$ and  $\hE_{t,x}[\cdot] := \hE[\cdot|X_t=x]$,  then by  \eqref{eq:definitionofh} we can write $h(t,x)=\hE_{t,x}[g(X_T)]$, where $X$ is the solution to (\ref{SDE0}) with $X_0=x_0$.
By  Feynman-Kac formula, we see that $h$  satisfy the PDE: 
\bea
\label{pde}
\begin{cases}
\partial_t h  (t,x)+ \scL_t h (t,x)=0;\\
h(T,x) = g(x)=\frac{f_{\rm tar}(x)}{p(T,x;0, x_0)},
\end{cases}
\eea
where the infinitesimal generator $\scL_t$ is defined by $\scL_t := b(t,x)\cdot \nabla  + \frac{1}{2} \Delta$. Similarly, we define $g_k(x):= \frac{f_k(x)}{p(T,x;0, x_0)}$, then the function
$h^k(t,x)$ can also be represented as the solution of the PDE:
\bea
\label{pdek}
\begin{cases}
\partial_t h^k  (t,x)+\scL_t h^k (t,x)=0;\\
h^k (T,x) =g_k(x)= \frac{f_k(x)}{p(T,x;0, x_0)}.
\end{cases}
\eea
Another useful observation is that for $\h\m_k\in \cS_k$, we have $kG(\h\m_k)\le D_k(\h\m_k)\le m$, or 
\begin{equation}
\label{Gbound}
G(\h\m_k)\le m/k.
\end{equation}
With the above observations, the key is to establish the convergence rate of $|h(t,x)-h^k(t,x)|$ and $|\nabla h(t,x)-\nabla h^k(t,x)|$ as $k$ tends to infinity.

\subsection{The Convergence of Optimal Policies}\label{sec:conv_opt_policy}
We shall now argue that the   optimal controls for Problem \ref{softSBP}, $\{\h\a^k(\cd,\cd)\}$, given by Proposition \ref{existak}, actually converges to the solution of the original SBP $\h\a(\cd,\cd)$ given by Lemma \ref{lemma:sdb}, and also establish its rate of convergence. More precisely, we have the following theorem.

\begin{Theorem}
\label{thm:convergence_opt_policy}
Assume that the Assumption \ref{assump:general penalty term} is in force, and that $\mi = \delta_{x_0}$ for some $x_0\in \mathbb{R}^d$.
Furthermore, assume that there exists  constants $C, \delta >0$, such that $\d\le g(x),g_k(x)\leq C$, $x\in \hR^d$, $k\in\hN$. Let $\widehat\alpha(t,x)$ and $\widehat\alpha^{k}(t,x)$, $(t,x)\in[0,T]\times\hR^d$ be the optimal controls given in Lemma \ref{lemma:sdb} and Proposition \ref{existak}, respectively. Then, it holds that 
\begin{equation*}
	\int_0^T |\widehat\alpha^k(t,x) - \widehat\alpha(t,x)| dt   \leq \frac{C} {k}, \qq x\in\hR^d,
\end{equation*}
where $C>0$ is some constant  independent of $k$.
\end{Theorem}
\begin{Remark}
{\rm We  note that the assumption $\d\le g(x)=\frac{f_{\rm tar}(x)}{P(T,x;0,x_0)}\le C$ (resp. $\d\le g_k(x)\le C$) amounts to saying that $ f_{\rm tar}(x)$ (resp. $f_k(x)$) $\propto \, P(T,x;0,x_0)$ as $x\to \infty$, which is not particularly a stringent condition in light of the general estimate (\ref{aronson}), and the arbitrariness of the sample data selection for the data generation procedure.
	\qed}
\end{Remark}

\begin{proof} First, by definition $\widehat\a(t,x)=\nabla\log h(t,x)$ and $\widehat\a^k(t,x)=\nabla\log h^k(t,x)$, where $h^k$ and $h$ are the solution to (\ref{pde}) and (\ref{pdek}), respectively, and $\nabla=\pa_x$.  We can easily deduce that
\bea
\label{estalpha}
| \widehat\alpha^k(t,x) - \widehat\alpha(t,x)| &=& | \nabla \log  h^k(t,x) - \nabla \log h(t,x) |= \left | \frac{\nabla h^k(t,x)}{h^k(t,x)} - \frac{\nabla h(t,x)}{h(t,x)} \right |\\
&=&  \left| \frac{\nabla h^k(t,x) h(t,x)- \nabla h(t,x) 
	h(t,x) + \nabla h(t,x)  h(t,x)  -\nabla h(t,x)  h^k(t,x)}{h^k(t,x)h(t,x)} \right|\nonumber\\
&\le& \left| \frac{\nabla h^k(t,x) - \nabla h(t,x) }{h^k(t,x)} \right| + |\nabla h(t,x)| \left | \frac{h(t,x)  -  h^k(t,x)}{h^k(t,x)h(t,x)} \right|=:I_1+I_2.
\nonumber
\eea
We now estimate $I_1$ and $I_2$, respectively. To this end we first apply the well-known Bismut-Elworthy-Li formula \cite{Bismut, EL} (see also the representation formula in \cite{fournie1999applications, MaZhang1}) to get 
\bea
\label{bel}
\left\{\ba{lll}
\nabla h(t,x)=\pa_x \hE_{t,x}[g(X_T)]
= \hE_{t,x}\big[g(X_T) N_T\big], \ms\\
\nabla h^k(t,x)=\pa_x \hE_{t,x}[g_k(X_T)] 
= \hE_{t,x}\big[g_k(X_T) N_T\big],
\ea\right.
\eea
where, 
\beaa
\label{DNtx}
N_s=N^{t,x}_s:=\frac{1}{s-t}\int_t^s (\td X^{t,x}_r)^\top dW_r, \qq s\in[t,T],
\eeaa
and $\nabla X=\nabla X^{t,x}$ is a $\hR^{d\times d}$-valued variational process satisfying the (random) ODE:
\beaa
\label{DXtx}
\pa_{x^j} X^i_s = \d_{ij} +  \int_t^s\sum_{\ell=1}^d\pa_{x^\ell}b^i(r, X_r) \pa_{x^j} X^\ell_r dr, \qq 1\le i,j\le d, \q s\in[t,T].
\eeaa
Furthermore, 
one can easily check that
\beaa
\label{Ntxest}
\hE\big[|\td X^{t,x}_s|^2\big] \le Ce^{C(s-t)},\qq
\hE\big[|N^{t,x}_s|^2\big]\le {C\over s-t}e^{C(s-t)}, \qq 0\le t\le s\le T.
\eeaa
Therefore, denoting $C>0$ to be a generic constant that is allowed to vary from line to line, and applying Assumption \ref{condition: dominance} and estimate (\ref{aronson}) we have
\begin{align}\begin{split}
		| \nabla h^k(t,x) - \nabla h(t,x)| &\le\hE\big[ |g(X^{t,x}_T)-g_k(X^{t,x}_T)||N^{t,x}_T|\big]=\hE\Big[\Big|\frac{f_k(X^{t,x}_T)- f_{\rm tar}(X^{t,x}_T)}{p(T, X^{t,x}_T; 0, x_0)}\Big||N^{t,x}_T|\Big] \\
		&\leq  \Big(\hE[|N^{t,x}_T|^2] \Big)^{\frac{1}{2}} \Big[\hE\Big|\frac{f_k(X^{t,x}_T)- f_{\rm tar}(X^{t,x}_T)}{p(T, X^{t,x}_T; 0, x_0)}\Big|^2\Big]^{\frac{1}{2}}\\
		&\le \frac{Ce^{C(T-t)}}{\sqrt{T-t}}G(\widehat\m_k)\Big[\hE\Big|\frac{\phi(X^{t,x}_T)}{p(T,X^{t,x}_T; 0, x_0)}\Big|^2\Big]^{\frac{1}{2}}
		\le \frac{C}{k\sqrt{T-t}}.\label{nablahbound}
	\end{split}
\end{align}

Next, we note that by assumption $\d\le g(x), g_k(x)\le C$ for all $x\in\hR^d$ and $k\in\hN$, by the weak maximum principle we conclude that as the solutions to the PDEs (\ref{pde}) and (\ref{pdek}), respectively, it holds that $\d\le h(t,x), h^k(t,x)\le C$, for all $(t,x)\in \hR^d \times [0,T)$. 
Consequently, we have 
\begin{equation}
	\label{I1}
	I_1 \leq \frac{C}{\delta k\sqrt{T-t}} \leq  \frac{C}{ k\sqrt{T-t}}.
\end{equation}

Similarly, we can argue that $|\nabla h(t,x)|\le \frac{C}{\sqrt{T-t}}$, and that 
\bea
\label{hkhbound}
|  h^k(t,x) -  h(t,x) | \le 	\hE \Big[ \Big|\frac{f_k(X^{t,x}_T)- f_{\rm tar}(X^{t,x}_T)}{p(T, X^{t,x}_T;0,x_0)}\Big|\Big]
\leq CG(\widehat\m_k)\hE\Big[\Big| \frac{\phi(X^{t,x}_T)} {p(T,X^{t,x}_T;0,x_0)}\Big| \Big]
\le
\frac{C}{k}, 
\eea
where the last inequality is due to \eqref{Gbound} and hence $I_2 \leq \frac{C}{k\sqrt{T-t}}$. This, together with (\ref{I1}) and \eqref{estalpha}, we obtain
\bea
\label{eq:final-alpha-diff}
|\widehat\alpha^k(t,x) - \widehat\alpha(t,x)|\leq \frac{C}{k\sqrt{T-t}}
\eea
and hence convergence result :
\beaa
\int_0^T |\widehat\alpha^k(t,x) - \widehat\alpha(t,x)| dt \leq \int_0^T \frac{C}{k\sqrt{T-t}} dt  \leq \frac{C\sqrt{T}}{k},
\eeaa
proving the theorem.
\end{proof}

\begin{Remark}
{\rm A particular example is when we take the penalty function $G(\mu)=D_{\rm KL}(\mu\|\mt)$. In this case, it is known  
	(see, e.g., \cite[Theorem 2]{garg2024soft}) that the optimal control for \eqref{RSB_dynamics}-\eqref{RSB_objective} is given by
	$\widehat\alpha_t^k = \nabla \log h^k(X_t^{\widehat\alpha^k},t)$, where
	\begin{eqnarray*}
		h^k(t,x) = d_k^{-1} \int p(T,z;t,x) \Big(\frac{f_{\rm tar}(z)}{p(T,z;0,x_0)}\Big)^{\frac{k}{k+1}} d z,
	\end{eqnarray*}
	with $d_k = \int f_{\rm tar}(x)^{\frac{k}{1+k}}p(x,T;\,x_0,0)^{\frac{1}{1+k}}d x$. 
	Consequently,   Assumption \ref{assump:general penalty term}-(ii) can be reduced to that $\hE \Big[\big|\frac{f_{\rm tar}(X_T)}{p(T, X_T;0,x_0)}\big|^2\Big]$ is bounded (see Assumption \ref{assum:g} for similar conditions); and the linear rate of convergence can be proved with the same arguments.
	\qed}
\end{Remark}

\subsection{The Convergence of the Value Function}

Having worked out the convergence analysis for the optimal controls, it is natural to extend the results to the convergence of value functions. However, the singularity at the terminal time $T$  in (\ref{eq:final-alpha-diff}) requires some technical care. It turns out that the popular notion of {\it early stopping} in diffusion models as well as the flow-based method literature \cite{bai2021understanding,li2023generalization,han2024neural} is exactly the remedy to this issue.

To be more precise, for any $\e>0$, we introduce the following $\varepsilon$-value function.
\begin{eqnarray*}
J_{\varepsilon}(\alpha) := \mathbb{E}\left[\int_0^{T-\varepsilon}\frac{1}{2}|\alpha|^2dt\right].
\end{eqnarray*}
There are many practical reasons, mainly for computational purposes, to invoke the notion of early stopping, as elaborated in
\cite{bai2021understanding,li2023generalization,han2024neural}. But on the other hand, it is clear that the  $\e$-value function effectively excludes the singularity at the terminal time $T$. This leads to the following straightforward result.

\begin{Proposition}
\label{ValueRate}
Assume that all the assumptions of Theorem \ref{thm:convergence_opt_policy} are in force. Then, for any $\e>0$, there exists a generic constant $C:=C(\varepsilon)=\mathcal{O}(\frac{1}{\sqrt{\varepsilon}})>0$, independent of $k$, such that
\begin{eqnarray}\label{eq:value-bound}
	|J_\varepsilon(\widehat\alpha^k)-J_\varepsilon(\widehat\alpha)|\leq \frac{C}{k}, \qq k\in\hN.
\end{eqnarray}
where $\h\a^k$ and $\h\a$ are the optimal controls in Theorem \ref{thm:convergence_opt_policy}, respectively.	
\end{Proposition}

\begin{proof}
The proof is straightforward. For any $k\in\hN$, let $\h\a^k$ and $\h\a$ be the optimal controls in Theorem \ref{thm:convergence_opt_policy}, respectively. Then, for any $\e>0$, applying \eqref{eq:final-alpha-diff} we have 
\begin{eqnarray}
	\label{eq:inter11}
	|J_\varepsilon(\widehat\alpha^k)-J_\varepsilon(\widehat\alpha)| \neg&\neg\leq\neg&\neg  \mathbb{E}\Big[\frac{1}{2}\int_0^{T-\varepsilon}\big| |\widehat\alpha^k_s|^2-|\widehat\alpha_s|^2\big|d s \Big] \le \mathbb{E}\Big[\frac{1}{2}\int_0^{T-\varepsilon} \big|\widehat\alpha^k_s-\widehat\alpha_s\big|\big(|\widehat\alpha^k_s| + |\widehat\alpha_s|\big)  d s \Big]\nonumber\\
	\neg&\neg\leq\neg&\neg{\frac{C}{k}}\mathbb{E}\Big[\frac{1}{2}\int_0^{T-\varepsilon} \frac{1}   {\sqrt{T-s}}\big(|\widehat\alpha^k_s| + |\widehat\alpha_s|\big)d s \Big]\nonumber\\&\le& 
	\frac{C}{k\sqrt{\varepsilon}}\mathbb{E}\Big[\frac{1}{2}\int_0^T\big(|\widehat\alpha^k_s| + |\widehat\alpha_s|\big)\, d s \Big],
\end{eqnarray}
where 
the last inequality is due to the fact that $\frac{1}{\sqrt{T-s}}\le \frac{1}{\sqrt{\varepsilon}}$ for $s\in[0,T-\varepsilon]$. To further bound \eqref{eq:inter11}, we recall that the definitions of $J(\cd)$ \eqref{SB-objective} and $J^k(\cd)$ \eqref{RSB_objective}, $k\in\hN$,  and define
\beaa
\label{V*}
V^*= J(\widehat\alpha) = \inf_{\alpha\in \mathcal{A}}J(\alpha); \qq V^{k,*} = J^k(\widehat\alpha^k) = \inf_{\alpha\in \mathcal{A}} J^k(\alpha).
\eeaa
We should note that $X^{\widehat\a}$ follows the constrained dynamics \eqref{SB-dynamics}, whereas $X^{\widehat\a^k}$ follows the soft-constrained dynamics \eqref{RSB_dynamics}.  Clearly, by definition (\ref{RSB_objective}) we have
\beaa
\label{supJk}
\sup_{k \ge 1} J^k(\alpha) = 
\begin{cases}
	\dis	\mathbb{E}\Big[\int_0^T \frac{1}{2}|\alpha_t|^2\dd t \Big] &\textrm{ if } \,\, \hP_{X_T^\alpha} = \mt\\
	\infty  & \textrm{otherwise}.
\end{cases}
\eeaa
Thus, since $\widehat\a$ satisfies the constrained dynamics (\ref{SB-dynamics}), we have 
\bea 
\label{eq:value_inequality}	
V^* =J(\widehat \a) = \inf_{\alpha\in \mathcal{A}} \sup_{k \ge 1} J^k(\alpha)\ge  \inf_{\alpha\in \mathcal{A}}  J^k(\alpha)= J^k(\widehat\alpha^k) =V^{k,*}, \qq k\in\hN.   
\eea 
Consequently, we have, for each $k\in\hN$, a simple application of Cauchy–Schwarz inequality and the fact (\ref{eq:value_inequality}) yields	\bea
\label{eq:inter1}
\mathbb{E}\Big[\frac{1}{2}\int_0^T \big(|\widehat\alpha^k_s| + |\widehat\alpha_s|\big)  d s \Big] 
\le\frac{{\sqrt{T}}}{2}\Big(\mathbb{E}\Big[\int_0^T |\widehat\alpha^k_s|^2 d s \Big]\Big)^{1/2} + \Big(\mathbb{E}\Big[\int_0^T |\widehat\alpha_s|^2  d s \Big]\Big)^{1/2}
\le \sqrt{T} \sqrt{V^*}. 
\eea
Combining  \eqref{eq:inter11} and \eqref{eq:inter1}, we obtain \eqref{eq:value-bound}.
\end{proof}

Besides the convergence of the value functions, another important convergence, that relies crucially on the convergence of the optimal controls, is the convergence of the terminal law $\hP_{X^{\h\a^k}_T}$ (with respect to the target distribution $\mt$), measured, for instance, in the Wasserstein distance. Again, to avoid the technicalities that the singularity at terminal time $T$ might cause, we shall focus on the early stopped state $X^{\h\a^k}_{T-\e}$, which is a commonly used criterion in statistical estimation results for generative diffusion models (see, e.g., \cite{fu2024unveil,han2024neural,chen2025diffusion}).	More precisely, we have the following result.
\begin{Proposition}
\label{prop:3.9}
Let all assumptions in Theorem \ref{thm:convergence_opt_policy} be in force. Assume further that the optimal policy $\widehat\alpha$ of the original SBP is Lipschitz in $x$: there exists $\k>0$, such that
\begin{eqnarray}
	\label{eq:optimal_policy_lip}
	|\widehat\alpha(t,x)-\widehat\alpha(t,y)|\leq \kappa|x-y|, \qq t\in[0,T].
\end{eqnarray}
Then there exists a constant $C>0$, depending on the Lipschitz constants $L$ in Assumption \ref{assump1} and $\k$ in (\ref{eq:optimal_policy_lip}), but independent of $k\in\hN$, such that for any $\varepsilon>0$, it holds that
\begin{eqnarray}
	\label{eq:distribution_appx}
	{W}_2(\hP_{X^{\widehat\alpha^k}_{T-\varepsilon}},\mt) \leq \frac{C \sqrt{\ln T-\ln \varepsilon}}{k} + C \varepsilon.
\end{eqnarray}
In particular, if we choose $\varepsilon=\frac{1}{k}$, then it holds that
\begin{eqnarray}
	\label{eq:distribution_appx2}
	{W}_2(\hP_{X^{\widehat\alpha^k}_{T-\varepsilon}},\mt) \leq \frac{C}{k}\big(\sqrt{\ln k}+\sqrt{\ln T}) + 1\big) = \mathcal{O}\Big(\frac{\sqrt{\ln k}}{k}\Big).
\end{eqnarray}
\end{Proposition}

\begin{Remark}
\label{rem:3.10}
{\rm (i) The linear (i.e., $\sim \frac1k$) ``closeness'' between the law of the optimal state and $\mt$ has appeared several times so far. For example, \eqref{Gbound} implies that $G(\hP_{X^{\widehat\alpha^k}_T})=G(\hP_{X^{\widehat\alpha^k}_T};\mt)\leq \frac{c}{k}$, and by 
	Remark \ref{rem:3.3}-(ii), this implies that $W_1(\hP_{X^{\widehat\alpha^k}_T},\mt)\sim \frac1k$. The result  in \eqref{eq:distribution_appx2} is in the same spirit, by under the stronger ${W}_2$-distance, but compensated by an early stopping.

	\ss
	(ii) The Lipschitz condition (\ref{eq:optimal_policy_lip}) for the optimal control $\h\a$ is not unusual in the diffusion model literature (see, e.g., \cite{tang2024fine,chen2023score,chen2025diffusion}). In fact, this can be argued via regularity of the solution to the PDE (\ref{pde}) combined with the speed of decay of the density  $f_{\rm tar}$, which can be assumed and analyzed rigorously (see Assumption \ref{assum:g} below). We therefore consider such an assumption non-stringent.
	\qed}
\end{Remark}

\no [{\it Proof of Proposition \ref{prop:3.9}}.]
First note that $X^{\h\a}$ and $X^{\h\a^k}$ satisfy the following SDEs, respectively:
\bea
\label{SDE3.1}
\begin{cases}
d X^{\widehat\alpha}_t = [b(t,X^{\widehat\alpha}_t)  + \widehat\alpha_t(X^{\widehat\alpha}_t)]d t  +d W_t\qq\q &X^{\h\a}_0=x_0;\\
d X^{\widehat\alpha^k}_t = [b(t,X^{\widehat\alpha^k}_t) + \widehat\alpha_t^k(X^{\widehat\alpha^k}_t)]d t  +d W_t, &X^{\h\a^k}_0=x_0.			\end{cases}
\eea
Let us now denote $\h\a_t(x)=\h\a(t,x)$, $\h\a^k_t(x)=\h\a^k(t,x)$, and define 
$$ b^{\h\a}(t,x)=b(t,x)+\h\a_t(x),  \q  \D \h\a^k_t(x)=\h\a^k_t(x)-\h\a_t(x),  \qq (t,x)\in[0,T]\times\hR^d.
$$
Then we see that SDE (\ref{SDE3.1}) can be written as 
\beaa
\label{SDE3.2}
\begin{cases}
d X^{\widehat\alpha}_t = b^{\h\a}(t,X^{\widehat\alpha}_t)d t  +d W_t\qq\q &X^{\h\a}_0=x_0;\\
d X^{\widehat\alpha^k}_t = [b^{\h\a}(t,X^{\widehat\alpha^k}_t) + \D\widehat\alpha_t^k(X^{\widehat\alpha^k}_t)]d t  +d W_t, &X^{\h\a^k}_0=x_0.
\end{cases}
\eeaa
That is, 
\begin{eqnarray*}
X^{\widehat\alpha}_t-X^{\widehat\alpha^k}_t = \int_0^t[b^{\h\a}(s,X^{\widehat\alpha}_s)-b^{\h\a}(s,X^{\widehat\alpha^k}_s)+ \D\widehat\alpha_s^k(X^{\widehat\alpha^k}_s)] ds, \qq t\in[0,T]
\end{eqnarray*}
Note that by Assumption \ref{assump1} and (\ref{eq:optimal_policy_lip}), $b^{\h\a}$ is uniform Lipschitz in $x$ (with Lipschitz constant $L+\k$),
and applying the estimate \eqref{eq:final-alpha-diff}, we deduce easily that
\bea
\label{L2est0}
\mathbb{E}[|X^{\widehat\alpha}_t-X^{\widehat\alpha^k}_t|^2] \leq 2 T\int_0^t\Big[(L + \kappa)^2 \mathbb{E}[|X^{\widehat\alpha}_s-X^{\widehat\alpha^k}_s|^2]ds +  \frac{2c^2}{k^2} \ln\Big[ \frac{T}{T-t}\Big].
\eea
In what follows let us denote $C>0$ to be a generic constant depending only on $L$, $\k$, $c$, but independent of $k$, and we allow it to vary from line to line. Then, by a simple calculation using Gronwall's inequality, we see that (\ref{L2est0}) lead to that  
\bea
\label{L2est1}
\mathbb{E}[|X^{\widehat\alpha}_t-X^{\widehat\alpha^k}_t|^2] \leq 
\frac{C}{k^2} (\ln T - \ln (T-t))e^{C t}, \qq t\in[0,T).
\eea
Furthermore, for any $\e>0$, using the monotonicity of the log function we  deduce from (\ref{L2est1}) that
\beaa
\label{eq:u_bound}
\mathbb{E}[|X^{\widehat\alpha}_{T-\e}-X^{\widehat\alpha^k}_{T-\e}|^2] \leq 
\frac{C}{k^2} (\ln T-\ln \varepsilon) 
\eeaa
It then follows that 
\beaa
\label{eq:W-L2}
W_2 (\hP_{X^{\widehat\alpha}_{T-\e}}, \hP_{X^{\widehat\alpha^k}_{T-\e}})\leq \mathbb{E}[|X^{\widehat\alpha}_{T-\e}-X^{\widehat\alpha^k}_{T-\e}|^2]^{1/2}\le \frac{\sqrt{C}}{k}\sqrt{\ln T-\ln \e}.		\eeaa
Finally, since the function $b^{\h\a}=b+\widehat{\alpha}$  is Lipschitz, by standard $\hL^2$-continuity result of SDE, we have 
\beaa
\label{L2est2}
{W}_2(\hP_{X^{\widehat\alpha}_{T-\varepsilon}}, \hP_{X^{\widehat\alpha}_{T}}) \leq C  \varepsilon,
\eeaa
and consequently, noting that $\hP_{X^{\widehat\alpha}_{T}}=\mt$, we obtain
\beaa
{W}_2(\hP_{X^{\widehat\alpha}_{T-\varepsilon}},\mt) \le   W_2 (\hP_{X^{\widehat\alpha}_{T-\e}}, \hP_{X^{\widehat\alpha^k}_{T-\e}}) +  {W}_2(\hP_{X^{\widehat\alpha}_{T-\varepsilon}}, \hP_{X^{\widehat\alpha}_{T}})\le \frac{\sqrt{C(\ln T-\ln \varepsilon)}}{k} + C\e,
\eeaa
proving (\ref{eq:distribution_appx}), whence the proposition.
\qed

}

\section{Stability of the Solutions to the SBP}
\label{Stability}

We note that all the results in the previous section are based on an important assumption: $\mi=\d_{x_0}$, for some $x_0\in\hR^d$. In this and the next section, we shall extend the results to more general  initial condition  $\mi\in \sP_2(\hR^d)$ with density $f_{\rm ini}$, and establish a similar rate of convergence.

We shall begin an important aspect in probability theory, which is the basis for the so-called stability issues of the solutions to the classic Schr\"odinger bridge problem. For notational convenience,  we still denote $p(\cd, \cd;\cd,\cd)$ to be the  transition density of a standard $\hR^d$-valued diffusion (\ref{SDE0}). We begin with the following well-known result in diffusion theory (cf. e.g., \cite{beurling1960automorphism}).

\begin{Proposition}[\cite{beurling1960automorphism}]
\label{product_meas}    
For any $\mi, \m \in\sP(\hR^d)$, there exists a unique pair of $\sigma$-finite measures $ \nu_0,\nu_T\in \sM(\hR^d)$ such that the measure $\pi$ on $\hR^d \times \hR^d$ defined by
\bea
\label{measmu}
\pi(E)= \int_E p(T, y;0, x) \nu_0(dx)\nu_T(dy), \qq E\in\sB(\hR^d\times \hR^d)
\eea
has marginals $\mi$ and $\m$. Furthermore, $\nu_T$ and $\m$ (resp. $\nu_0 $ and $\mi$) are mutually absolutely continuous, 
denoted by $\nu_T\simeq \m$ (resp. $\nu_0 \simeq \mi$). 
\qed
\end{Proposition}
Fix $\mi\in\sP(\hR^d)$ apriori. Let us denote a (well-defined) mapping 
$\cT:\sP_2(\hR^d)
\to \sM(\hR^d)\times \sM(\hR^d)$ by $\cT(
\m)=(\n_0, \n_T)$. Note that in Proposition \ref{product_meas} the measures $(\n_0, \n_T)$ are only $\si$-finite in general, to facilitate our discussion,  we shall consider,  the following set:\beaa
\sD_{\mi}:=\{\m\in\sP_2(\hR^d):\cT(\m)\ll Leb(\cd); \, \cT(\m) (\hR^d\times\hR^d)<\infty\}.
\eeaa
Here $Leb(\cd)$ denotes the Lebesgue measure on $\hR^d\times\hR^d$.  

We note that if $\m\in\sD_{\mi}$ and  $(\n_0,\n_T)=\cT(\m)$, then
$\n_T$ must have a density function, which we shall denote by $
\rho^\m\in \hL^1(\hR^d)$.
Moreover, we define an operator $S:\sP_2(\hR^d)\to \sP_2(\hR^d)$ by 
\bea
\label{Sm}
S[\mu] (dy) = \int_{\hR^d}  p(T, y;0, x) \mu(dx)dy, \qq \m\in\sP_2(\hR^d).
\eea
Clearly, if $\m\in\sP(\hR^d)$, then $S[\m](dy)=f_{X^{0,\m}_T}(y)dy$, where $X^{0,\m}=\{X^{0,\m}_t\}_{t\in[0,T]}$ denotes the solution to (\ref{SDE0}) with $X^{0,\m}_0\sim\m$.
But the operator $S$ can be naturally extended to any $\m\in \sM(\hR^d)$, provided the right-hand side of (\ref{Sm}) is well-defined.

Let us now recall a well-known analogue of Lemma \ref{lemma:sdb} in the case of general initial condition  $\mi\in\sP_2(\hR^d)$. 

\begin{Proposition}[{\cite[Theorem 3.2]{dai1991stochastic}}]
\label{lemma:sdb_general_init}
Let $\mi,  \mt \in \sP_2(\hR^d)$ and $ (\nu_0,\nu_T)=\cT(\mt)$.  Assume that $D_{\rm KL}(\mi\| \nu_0)<\infty$ and $D_{\rm KL}(\mt\| S[\n_0])<\infty$.  Then, the optimal control for the (original) SBP \eqref{SB-objective}-\eqref{SB-dynamics}  is given by $\widehat\alpha_t = \nabla \log h(t, X^{\widehat\alpha}_t)$  where, denoting $\rho^\mt(\cdot)$ to be the density function of $\nu_T$, 
\begin{eqnarray}
\label{eq:definitionofh1}
h(t,x) := \int_{\hR^d} p(T,z;t,x) \rho^\mt(z) dz.
\end{eqnarray}
Moreover, it holds that
\begin{equation}
\label{valuefunction_schrodinger_general_init}
J(\widehat\alpha) = \int_{\hR^d} \log \rho^\mt(y) \mt (dy) - D_{\rm KL}(\mi\|\nu_0).
\end{equation}
\end{Proposition}

We note that in the above $D_{\rm KL}(\mi\|\n_0)=\int \log\frac{\mi(dx)}{\n_0(dx)}\mi(dx)$ (see footnote 1), and (\ref{measmu}) implies that $\frac{\mi(dx)}{\n_0(dx)}=\int p(T, y;0,x)\rho^\mt(y)dy$. Therefore \eqref{valuefunction_schrodinger_general_init} can be rewritten as
\begin{align*}
J(\widehat\alpha) &= \int_{\hR^d} \log \rho^\mt(y) \mt(dy) - \int_{\hR^d} \log\Big(\int_{\hR^d}p(T, y;0,x)\rho^\mt(y)dy\Big) \mi(dx) \\
&= \hE[\log\rho^\mt(X^{\widehat\alpha}_T)]- \int_{\hR^d} \log h(0,x) \mi(dx) = \hE\big[\log\rho^\mt(X^{\widehat\alpha}_T)\big]- \hE[\log h(0,X^{\widehat\alpha}_0)].
\end{align*}

Moreover, for a fixed $\mi\in \sP_2(\hR^d)$, 
we define $h^\m(t,x)
= \int_{\hR^d} p(T,z;t,x) \rho^\m(z) dz$. Then, we have the following result.

\begin{Lemma}[{\cite[Lemma 3.1]{garg2024soft}}]
\label{lemma:3.1}
Let $\m\in{\sD_{\mi}}$. Then, for any $\{\alpha_t\}\subset \hL^2_{\hF^0}([0,T];\hR^d)$, it holds that
\begin{align*}
J(\alpha) \geq \hE[\log \rho^\m(X^\alpha_T)]- \hE[\log h^\m(0, X^\alpha_0)].
\end{align*}
The equality holds when $\alpha_t =\a^\m_t= \nabla \log h^\m(t, X^{\alpha^\m}_t)$, $t\in[0,T]$ and $X^{\a^\m}_T\sim \m$.
\qed
\end{Lemma}

From Proposition \ref{lemma:sdb_general_init} and Lemma \ref{lemma:3.1} we see that the density function $\rho^\m$ plays an important role in the structure of the solution of SBP. We shall be particularly interested in the continuous dependence of $\rho^\m:=\G_1(\m)$ on $\m\in\sP_2(\hR^d)$, which we shall refer to as the {\it Stability of the 
SBP}, borrowing the well-known concept of the SBP theory (cf. e.g., \cite{nutz2022entropic,divol2025tight,carlier2024displacement}). 

To continue our discussion, we shall identify a set $\sE\subset \sP_2(\hR^d)$ 
on which an argument based on Schauder's fixed-point theorem can be carried out. We begin by  denoting
\beaa
\label{SK}
\cK := \Big\{\m\in\sP_2(\hR^d): \mbox{$\m$ has density $f_\m\in\hL^1(\hR^d)$} \Big\}.
\eeaa
Furthermore, we shall make use of the following assumption. 
\begin{Assumption}
\label{assum:g}
(i) There exists a constant $C>0$, independent of $\m\in\sP_2(\hR^d)$, such that $|\log f_\m(x)|\leq C|x|^2$, $x\in\hR^d$;

\ms
(ii) There exists a function $g\in\hL^2(\hR^d;(0,1])$, with $\int_{\hR^d}|x|^2g(x)ds <\infty$, and a constant $K>0$, independent of $\m\in\sP_2(\hR^d)$, such that		
$\Big\|\frac{f_\m}{g^2}\Big\|_{\infty} \leq K$.
\end{Assumption}
We shall consider the following two sets that will play a crucial role in our discussion.
\begin{eqnarray}\label{sE}
\sE:=\left\{\m\in\cK: \text{ Assumption  \ref{assum:g} holds} \right\}\subset \sP_2(\hR^d);
\quad
\cS_{\sE}: = \left\{ f_\m: \m\in \sE\right\}\subset \hL^1(\hR^d).
\end{eqnarray}

\begin{Remark}
\label{rem:g}
We note that  Assumption \ref{assum:g}-(i) provides a uniform lower bound for the densities $f_\m\in\cK$, whereas  Assumption \ref{assum:g}-(ii) controls the decay of $f_\m\in\cK$ as $x\sim\infty$, hence mutually non-inclusive. A typical example of the function $g$ is $e^{-c|x|^2}$, $x\in\hR^d$, $c>0$. 
In fact, in light of the estimate (\ref{aronson}), such a property holds essentially for all transition probabilities of diffusion processes.
\qed
\end{Remark}

The following lemma lists some basic properties of the set $\sE$ (or Assumption  \ref{assum:g}).  
\begin{Lemma}
\label{lem:cK}
Assume that Assumption \ref{assum:g} is in force. Then it holds that

(i) The set $\{f_\m\}_{\m\in\sE}$ is uniformly bounded  in $\hL^2(\hR^d)$.

(ii) The set $\{f_\m\}_{\m\in\sE}$ is uniformly integrable in $\sP_2(\hR^d)$, in the sense that
\bea
\label{UI}
\lim_{R\to\infty}\sup_{\m\in\sE}\int_{\{|x|\ge R\}}|x|^2f_\m(x)dx=0.
\eea

(iii) If $\{\m_n\}_{{n\ge1}}\subset\sE$ such that $\m_n\Rightarrow\m$, as $n\to\infty$, then $\|f_{\m_n}- f_\m\|_{\hL^1}\to 0$.
\end{Lemma}

\begin{proof}

For any $\m\in\sE$, we note that $0<g(x)\le 1$, and by assumption,
$$ \int_{\hR^d} |f_\m(x) |^2dx \le K^2\int_{\hR^d}|g (x)|^4ds\le   K^2\|g\|^2_{\hL^2},
$$
That is $\{f_\m\}_{\m\in\sE}$ is uniformly bounded (by $K\|g\|_{\hL^2}$) in $\hL^2(\hR^d)$, proving (i).

Similarly, for any $\m\in\sE$, by the absolute continuity of the integral we have
\beaa
\label{mutoinf}
\sup_{\m\in\sE}\int_{\{|x|\ge R\}}|x|^2f_\m(x)dx \le K  \int_{\{|x|\ge R\}}|x|^2g(x)dx \to 0,
\q \mbox{\rm as $R\to\infty$,}
\eeaa
This proves (\ref{UI}), whence (ii). 

The proof of part (iii) is slightly more involved, which is in the spirit of the so-called Scheff\'e's theorem (cf. \cite{sweet}).  
We note that $\m_n\Rightarrow \m$ amounts to saying that $ f_{\m_n}\limw f_\m$, as $n\to\infty$, in $\hL^2(\hR^d)$. To show $f_{\m_n} \to f_\m$ in $\hL^1(\hR^d)$, we first consider, for each $m>0$, the smooth mollifiers $\f^m\in\hC^\infty(\hR^d;\hR_+)$ such that $\int_{\hR^d}\f^m(z)dz=1$, $m\ge 1$, and denote 
$$f^m_{\m_n}(x)=[\f^m * f_{\m_n}](x)=\int_{\hR^d}\f^m(x-z)f_{\m_n}(z)dz; \q f^m_{\m}(x)=[\f^m*f_\m](x), \q x\in\hR^d.
$$
Then it is clear that for each $n\in\hN$, $\lim_{m\to\infty}f^	m_{\m_n}(x)=f_{\m_n}(x)$ and $\lim_{m\to\infty}f^m_{\m }(x)=f_{\m}(x)$, for a.e. $x\in\hR^d$. We should remark that the convergence is uniform in $n$. Indeed, by Assumption \ref{assum:g} and Dominated Convergence Theorem we have,  as $m\to\infty$, for all $n\ge 0$, 
and $x\in\hR^d$, 
$$  |f^m_{\m_n}(x)-f_{\m_n}(x)|\le \int_{\hR^d}|\f^m(x-z)-\d_{x}(z)|f_{\m_n}(z)dz\le K\int_{\hR^d}|\f^m(x-z)-\d_{x}(z)|g^2(z)dz\to 0.
$$
Furthermore, since $\sup_{\m\in\sE}|f_\m|\le Kg^2\in\hL^1(\hR^d)$, by Dominated Convergence Theorem we have $\lim_{m\to\infty}f^m_{\m_n}=f_{\m_n}$ in $\hL^1(\hR^d)$, uniformly for $n\ge 0$.  That is, for any $\e>0$, there exists
$M(\e)>0$, such that for all $n\ge 1$, it holds that
\bea
\label{3e}
\|f^m_{\m_n}-f_{\m_n}\|_{\hL^1}<\frac{\e}3; \q \|f^m_{\m}-f_{\m}\|_{\hL^1}<\frac{\e}3, \q \mbox{whenever $m>M$.}
\eea
In the sequel we fix $m>M(\e)$, and take a closer look at the sequence $\{f^m_{\m_n}\}_{n\ge1}$. Clearly, each $f^m_{\m_n}$ is still a density function, and it holds that
\bea
\label{bdd}
\sup_n|f^m_{\m_n}(y)|\le \sup_n[\f^m*|f_{\m_n}|](y)\le K.
\eea
Moreover, since $\f^m$ is   continuous, thus for any $x, y\in\hR^d$, 
applying the Dominated Convergence Theorem we have
\beaa
\label{aec}
|f^m_{\m_n}(x+y)-f^m_{\m_n}(x)|&\le& \int_{\hR^d}|\f^m(x+y-z)-\f^m(x-z)||f_{\m_n}(z)|dz\nonumber\\
&\le& K\int_{\hR^d}|\f^m(z'+y)-\f^m(z')|dz'\to 0, \q\mbox{as $y\to 0$.}
\eeaa
Clearly, the convergence above is uniform in $n$. That is, the sequence $\{f^m_{\m_n}\}_{\{n\ge1\}}$ is so-called {\it asymptotically equi-continuous} in the sense of Sweeting \cite{sweet}. This, together with (\ref{bdd}), implies that $\dis\lim_{n\to\infty}f^m_{\m_n}= f^m_\m $, uniformly on compacts in $\hR^d$  (cf. \cite[Theorem 1]{sweet}). Applying the Dominated Convergence Theorem again we have $\lim_{n\to\infty}\|f^m_{\m_n} - f^m_\m\|_{\hL^1}=0$. That is, for the given $\e>0$ in (\ref{3e}), there exists $N>0$ such that
$\|f^m_{\m_n} - f^m_\m\|_{\hL^1}<\frac{\e}3$,   whenever $n>N$. This, together with (\ref{3e}), yields
$$ \|f_{\m_n}-f_\m\|_{\hL^1}\le \|f_{\m_n}-f^m_{\m_n}\|_{\hL^1}+ \|f^m_{\m_n} -f^m_{\m}\|_{\hL^1}+\|f^m_{\m}-f_{\m}\|_{\hL^1} <\frac{\e}3+\frac{\e}3+\frac{\e}3=\e, \q n>N,$$
proving (iii), whence the Lemma.
\end{proof}

We are now ready to study our main stability result. More precisely, we shall argue that the mapping $\G_1: \sP_2(\hR^d)\to \hL^1(\hR^d)$ is continuous. That is,  that $\mu_n\in\sP_2(\hR^d)$ weakly converges to some $\mu\in\sP_2(\hR^d)$ in Prohorov metric would imply that $\rho^{\m_n}$ converges to $\rho^{\mu}$ in $\hL^1(\hR^d)$. Such a result, to the best of our knowledge, is novel in the literature. 

To simplify our discussion, in what follows, we assume $T=1$.  Recall that $\pi$ in \eqref{measmu} has marginals $\mi$ and $\m$, and in what follows, we shall assume that $\mi$ is fixed and $\mu\in \sE$.  Let us now consider the following {\it  entropic optimal transport problem}:
\begin{equation}
\label{EOT}
I(\m):=\inf_{\pi \in \Pi(\mi,\mu)} \int_{ \hR^d \times \hR^d }\bc(x,y)\pi(dxdy) + D_{\rm KL}(\pi\|\mi\otimes\mu),
\end{equation}
where $\Pi(\mi,\m)$ is the set of all coupling probability measures $\pi$ on $\hR^d \times \hR^d$ with  marginals $\mi$ and   $\m$; and $\bc(\cd, \cd)$ is a continuous {\it cost function}.  It is well-known (see, e.g., \cite{divol2025tight,chiarini2023gradient,gunsilius2022convergence,nutz2022entropic}) that the minimization \eqref{EOT} admits a unique solution $\widehat\pi$, whose density takes the form:
\begin{equation}
\label{density_decomp}
\widehat\pi(dxdy) = \exp\big(-\bc(x,y)+\phi^\m(x)+\psi^\m(y)\big)\mi(dx)\m(dy),
\end{equation}
where $\phi^\m,\psi^\m$: $\hR^d\to \hR$ are two measurable functions, often referred to as the {\it Schr\"odinger potentials}. It is clear that the pair $(\phi^\m, \psi^\m)$ is unique up to an additive constant. That is, if $(\phi^\m, \psi^\m)$ is a pair of  Schr\"odinger potentials, then so is   $(\phi^{\m}+c,\psi^{\m}-c)$. Furthermore, since both $\mi$ and $\m$ are probability measures, we can easily choose a constant $c$ so that
the following {\it symmetric normalization} holds:
\bea
\label{symnorm}
\int \phi^{\m}(x) \mi(dx) = \int \psi^{\m}(y)\m(dy).
\eea
(Otherwise we take $c=\frac12\big[-\int \phi^{\m}(x) \mi(dx) +\int \psi^{\m}(y)\m(dy)\big]$.) Note that under the symmetric normalization, the Schr\"odinger potentials is unique. The following  stability result for the mappings $\m\mapsto (\phi^\m, \psi^\m)$ is crucial for our discussion.   
\begin{Lemma}[{\cite[Theorem 1.1]{carlier2024displacement}}]
\label{psicontrol}
Assume that the cost function $\bc(\cd,\cd)\in \hC^{k+1}(\hR^d\times\hR^d)$ for some $k\in\hN$. Then there exists $C>0$ depending only on $\|\bc\|_{\mathbb{C}^{k+1}}$, such that for all $\m_1$, $\m_2 \in \sP_2(\hR^d)$, it holds that
\begin{equation*}
\|(\phi^{\m_1}-\phi^{\m_2},\psi^{\m_1}-\psi^{\m_2})\|_* \leq CW_2(\mu_1, \ \m_2),
\end{equation*}
where $\|(\phi,\psi)\|_*:= \inf_{c\in\hR} \left\{\|\phi-c\|_{\hC^k(\hR^d)}+ \|\psi+c\|_{\hC^k(\hR^d)}\right\}$.
\qed
\end{Lemma}


We now proceed to prove the main result of this section. To begin with,  let us consider the entropic optimal transport problem (\ref{EOT}) with $\bc(x,y):=-\log p(1,y;0,x)$, $x,y\in\hR^d$, where $p(s,y;t,x)$, $0\le t<s\le 1$ and $x,y\in\hR^d$, is the transition density of the diffusion (\ref{SDE0}).  By  (\ref{density_decomp}), for fixed $\mi, \m\in\sP_2(\hR^d)$, the unique solution for this entropic optimal transport problem is given by (see also \cite{nutz2022entropic})
\beaa
\label{hatpi}
\widehat\pi(dxdy)= p(1,y;0,x)e^{\phi^\mu(x)+\psi^\mu(y)}f_{\rm ini}(x)f_{\mu}(y)dxdy,
\eeaa
where $\phi^\m$ and $\psi^\m$ are the Schr\"odinger potentials, and we shall enforce the symmetric normalization so that
they satisfy (\ref{symnorm}). Since $\h{\pi}$ has the marginals $\mi$ and $\mu$, by  the uniqueness of $(\n_0, \n_T)=\cT(\m)$, 
in Proposition \ref{measmu}, we can conclude that 
\begin{eqnarray*}
p(1,y;0,x) \rho_0(x)\rho^\mu(y) = p(1,y;0,x)e^{\phi^\mu(x)+\psi^\mu(y)}f_{\rm ini}(x)f_{\mu}(y), \qq x, y\in\hR^d,
\end{eqnarray*}
where $\rho_0$ is the density function of $\nu_0$.
An easy argument of separation of variables then yields that 
\bea
\label{rhopsi}
\rho^\mu(y)=e^{\psi^\mu(y)}f_\mu(y);\quad \rho_0(x)=e^{\phi^\mu(x)}f_{\rm ini}(x), \qq x, y\in\hR^d.
\eea

Now note that  the transition density $p(\cd,\cd\,;\cd,\cd)$ is a classical solution to the Kolmogorov PDE. Thanks to Assumption \ref{assump1}, we can assume without loss of generality that  $\bc(\cd, \cd)=-\log p(1,\cd\,;0,\cd)\in \hC^2(\hR^d\times \hR^d)$. Thus  
according to Lemma \ref{psicontrol} and noting the definition of $\|\cd\|_*$, we see that, modulo some constant normalization, we have  that the 
Schr\"odinger potential  $(\phi^{\m_n},\psi^{\m_n})$ itself satisfies the estimate:
\begin{equation}
\label{represent}
\|  {\phi}^{\m_n} - \phi^{\m}\|_{\hL^\infty} +  \|  {\psi}^{\m_n} - \psi^{\m}\|_{\hL^\infty}\leq CW_2(\m_n,\m).
\end{equation} 
Here in the above the constant $C>0$ depending only on $\|\bc\|_{\hC^2}$, but independent of $n$.

Furthermore, we note that $\bc\in\hC^2$ also lead to the following {\it a priori} estimate of the Schr\"odinger potential (see, e.g., \cite[Lemma 2.1]{nutz2022entropic}):
\bea
\label{est_potential}
\psi^{\mu}(y)\leq  \int_{\hR^d} \bc(x,y)\mi(dx) = :\xi(y), \q y\in\hR^d.
\eea
Recall the fundamental estimate (\ref{aronson}) and the definition of $\bc(\cd, \cd)$, it is readily seen that $|\xi(y)|\sim \l' |y|^2$, as $y\to \infty$, for some constant $\l'>0$ depending only on the coefficient $b(\cd,\cd)$ in SDE(\ref{SDE0}). In light of Remark \ref{rem:g} and \cite[Lemma 2.1]{nutz2022entropic}, in what follows we shall impose some explicit assumptions on the Schr\"odinger potentials $\{\psi^\m(\cd)\}_{\m\in\sE}$ and their upper bound $\xi (\cd)$ in (\ref{est_potential}), which are in line with Assumption \ref{assum:g}. 
\begin{Assumption}
\label{assum: psi}
(i) There exists a constant $C>0$, such that for any $\m \in\sE$, it holds that
\bea
\label{estpsi}
|\psi^\mu(y)|\leq C|y|^2,  \q y\in\hR^d.
\eea
(ii) In Assumption \ref{assum:g},  the control function $g$ satisfies
\bea
\label{cond_mu0}
\eta(\cd):=e^{\xi(\cd)}g^2(\cd) \in \hL^1(\hR^d).
\eea
\end{Assumption}

\ms

Now for any $f_\mu \in\cS_{\sE}$, by Assumption \ref{assum:g} and (\ref{est_potential}) we have
\begin{equation*}
0\le \rho^\m (y)=  e^{\psi^\m(y)}\ f_\m(y) \leq    e^{\xi(y)}f_\m(y)  \leq   e^{\xi(y)}Kg^2(y)=   K\eta(y), \q y\in\hR^d.
\end{equation*}
Consequently, we conclude that $\rho^\m \in \hL^1(\hR^d)$ for any $\m\in\sE$, thanks to \eqref{cond_mu0}.

Bearing the above discussion in mind, we are now ready to present the main result of this section. 
\begin{Proposition}
\label{stable}
Assume that Assumptions \ref{assump1} and  \ref{assum:g} are in force. Assume further that $\{\mu_n\}_{n\ge1}\subset \sE$ and $\m_n\Rightarrow \m $ in Prohorov metric. Then $\|\rho^{\m_n}-\rho^{\mu}\|_{\hL^1}=\|\G_1(\m_n)-\G_1(\m)\|_{\hL^1}\to 0$, as $n\to\infty$. 
\end{Proposition}
\begin{proof}

Assume $\{\m_n\}_{n\ge 1}\subset \sE$, and $\m_n\Rightarrow \m$, in Prohorov metric. By Lemma \ref{lem:cK}-(ii), $\{\m_n\}$ is uniformly integrable in $\hL^2$, thanks to Assumption \ref{assum:g}, and thus by the relationship between Wasserstein distance and Prohorov metric (see, \cite[Theorem 7.12]{Villani}), we have $W_2(\m_n,\m)\to 0$, as $n\to\infty$. Thus, if follows from (\ref{represent}) that $\|\psi^{\m_n}-\psi^\m\|_{\hL^\infty}\to 0$, as $n\to\infty$.

Next,  for each $\m_n\in\sE$, $n\ge 1$, and $\m$, we apply (\ref{rhopsi}) and  write
\begin{eqnarray*}
\rho^{\mu_n}(y)=e^{ \psi^{\mu_n}(y)}f_{\mu_n}(y),\quad 
\rho^{\mu }(y)=e^{ \psi^{\mu}(y)}f_{\m}(y), 
\qq x, y\in\hR^d.
\end{eqnarray*}
Therefore, for $y\in \hR^d$, we have
\beaa
|\rho^{\mu}(y) - \rho^{\mu_n}(y)|  &=& \big|e^{\psi^{\mu}(y)}f_{\mu }(y) - e^{\psi^{\mu_n}(y)}f_{\mu_n}(y)\big|\\
&\leq& \big|e^{\psi^{\mu}(y)}-e^{\psi^{\mu_n}(y)}\big| f_{\mu_n}(y) +e^{\psi^{\mu}(y)}\big|f_{\mu_n}(y)-f_{\mu}(y)\big| =:I^1_n(y)+I_n^2(y), \nonumber
\eeaa
where $I^i_n$, $i=1,2$ are defined in an obvious way. It then suffices to show that both $I^1_n$ and $I^2_n \to 0$ in $\hL^1$, as $n\to\infty$. 

To this end, we first recall that $\|\psi^{\m_n}-\psi^\m\|_{\hL^\infty}\to 0$, as $n\to\infty$. Hence there exists $N>0$, such that $\psi^{\m_n}(y)\leq \psi^{\m}(y)+1$, for all $y\in \hR^d$, whenever $n\geq N$.
Thus, for $n\geq N$, we have
\begin{equation*}
0\le I_n^1(y) \leq
\big(\big|e^{\psi^{\mu}(y)}\big|+\big|e^{\psi^{\mu_n}(y)}\big|\big)f_{\m_n}(y) \leq 2e^{\psi^{\m}(y)+1}f_{\m_n}(y)
\leq 2e \cd e^{\xi(y)}g^2(y)=2 e \eta(y).
\end{equation*}
Here in the above, the last inequality holds due to Assumption \ref{assum:g} and \eqref{cond_mu0}. Since $\eta \in \hL^1$ by (\ref{cond_mu0}), the Dominated Convergence Theorem implies that $I_n^1(\cdot)$ converges to $0$ in $\hL^1(\hR^2)$ as $n \to \infty$, because ${\psi^{\m_n}}$ converges uniformly to $\psi^{\m}$ on $\hR^d$.

Finally, since 
$I_n^2(y) \leq 2K\eta(y)$,  and  $f_{\m_n}\to f_\m$  in $\hL^1(\hR^d)$, thanks to Lemma \ref{lem:cK}-(iii), we can apply Dominated Convergence again to get $I_n^2(\cd)$ converges to 0 in $\hL^1(\hR^2)$, as $n\to\infty$, proving the proposition.
\end{proof}

\section{Existence of optimal control and convergence for general $\mi$}
\label{sec:general}

In this section, we shall extend the results of Section \ref{sec:existence_opt_policy} and show that the Problem \ref{softSBP} has solution for each $k\in\hN$ when $\mi$ is an arbitrary distribution with density $f_{\rm ini}$ in $\sP_2(\hR^d)$. To be more precise, for fixed $k\in\hN$, 
let $J^k(\alpha)$ be the cost functional  in Problem \ref{softSBP}. Applying Lemma \ref{lemma:3.1}, for any $\m\in\sD_\mi$, we have
\begin{align}
\begin{split}
\label{Jka}
J^k(\alpha)&=\hE \Big[ \frac{1}{2} \int_0^T |\alpha_s|^2 d s  + k G(\hP_{X^\alpha_T};\mt)\Big]\\
&\ge    k G(\hP_{X^\alpha_T};\mt) +  \hE[\log \rho^\m(X^\alpha_T)]- \hE[\log h^\m(0, X^\alpha_0)]
\end{split}
\end{align}
and the equality holds when $\alpha_t = \a^\m_t=\nabla \log h^\m(t, X^{\alpha^\m}_t)$ and $X^{\a^\m}_T\sim \m$.  
Our main goal of this
section is to determine the probability measure $\widehat \m$, such that $\widehat\alpha_t = \nabla \log h^{\widehat\m}(t, X^{\widehat\alpha}_t)$ is the optimal control to Problem \ref{softSBP}, where $h^{\h\m}(t,x) = \hE_{t,x}[ \rho^{\h\m}(X_T)]$.

We now give the heuristic idea of the construction of "solution mapping" $\G$. Let $\mi$ be given. For any $\m\in\sP_2(\hR^d)$, 
first apply Lemma \ref{lemma:3.1} to get the feedback control
$ \alpha^\mu_t = \nabla \log h^\mu(t,X^{\alpha^\mu}_t)$ so that  $\hP_{X^{\alpha^\mu}_T} = \mu$ and
\bea
\label{Jamu}
J(\alpha^\m) = \hE[\log \rho^\m(X^{\alpha^\m}_T)]- \hE[\log h^\m(0,X^{\alpha^\m}_0)].  
\eea
In what follows we fix $k\in\hN$. To find the $\widehat\m^k$ such that $J^k(\a^{\widehat\m^k})=\inf J^k(\a)$, we consider a mapping: 
$\G_2:\hL^1(\hR^d) \to \sP_2(\hR^d)$ by $\G_2(\rho^\m)=\m'$ where
\begin{equation}
\label{mu'}
\mu' = \argmin_{\bar\mu\in \sP_2(\hR^d)}\Big\{k G(\bar{\mu})  +\int_{\hR^d} \log \rho^\mu (y) \bar{\mu}(dy)
\Big\} .
\end{equation}
Finally, we define $\Gamma =\Gamma_2\circ\Gamma_1:\sP_2(\hR^d)\to\sP_2(\hR^d)$, and we shall argue that the mapping $\G$ has a fixed point $\widehat\m\in\sE$, where $\sE$ is defined by (\ref{sE}). Clearly, if $\G(\widehat \m)=\widehat\m$, then we can still define $\widehat\a=\a^{\widehat\m}$, and by Lemma \ref{lemma:3.1}  we have $\hP_{X_T^{\a^{\widehat\m}}}=\widehat\m$. Thus by (\ref{Jamu}), for any $\a \in\hL^2_{\hF^0}([0,T];\hR^d)$ we have 
\beaa
J^k(\widehat\a)&=&J(\widehat\a)+k G(\widehat\m) =k G(\widehat\m)  +\hE[\log \rho^{\widehat\m}(X^{\widehat\alpha}_T)]- \hE[\log h^{\widehat\m}(0,X^{\widehat\alpha}_0)]\\
&\le &k G(\hP_{X^\a_T})  +\hE[\log \rho^{\widehat\m}(X^{\alpha}_T)]- \hE[\log h^{\widehat\m}(0,X^{\alpha}_0)]\le J^k(\a).
\eeaa
Here in the above the first inequality is due to (\ref{mu'}), and the last inequality is due to (\ref{Jka}). This shows that $\widehat\a$ is the minimizer of $J^k(\cd)$.

We now show that the set $\sE\subset \sD_{\mi}$ defined by (\ref{sE}) satisfies all the necessary properties, thanks to the Lemma \ref{lem:cK} and Proposition \ref{stable} that we established in the last section, so that the mapping $\G$  possesses a fixed point on $\sE$ by Schauder's fixed-point theorem. 
Our main result is as follows.
\begin{Theorem}
\label{thm:sE}
Assume that Assumptions 
\ref{assum:g} 
is in force. Consider the set $\sE$  defined by (\ref{sE}). Then the following hold:

\ms
(i) $\sE$  is convex and closed under Prohorov metric, and $\cS_{\sE}$ is convex and closed in $\hL^1(\hR^d)$;

\ms
(ii) $\G(\sE)\subseteq \sE$, and is precompact in $\sP_2(\hR^d)$,  under both Prohorov and Wasserstein metric;

\ms

(iii) $\G$ is continuous on $\sE$, under Prohorov metric.

\ms
Consequently, the mapping $\G$ has a fixed point in $\sE$.
\end{Theorem}

\begin{proof} Since the last statement is a direct consequence of Schauder's fixed point theorem, applying to the space $\sP(\hR^d)$ with Prohorov metric,  we need only prove the properties (i)-(iii).

\ms
(i) is obvious.

\ms
(ii)  By definition of $\cS_{\sE}$ we rewrite \eqref{mu'} as
\begin{equation*}
\label{mu'1}
f_{\mu'}= \argmin_{f_{\bar\m}\in \cS_{\sE}}\Big\{k G({\bar\m})  +\int_{\hR^d} \log \rho^\mu (y) f_{\bar\m}(y)dy
\Big\} .
\end{equation*}
Since $\cS_{\sE}$ is  convex and closed in $\hL^1(\hR^d)$, it follows that $f_{\m'}\in \cS_\sE$, and thus $\G(\sE)\subseteq \sE$. 
We are to show that $\G(\sE)$ is precompact in $\sP_2(\hR^d)$.  

To this end, let $\{\m_n\}\subseteq \G(\sE)$ be any sequence, we shall find a subsequence $\{\m_{n_k}\}_{k\ge1}$ such that $\lim_{k\to\infty}\m_{n_k}=\m\in \sP_2(\hR^2)$, under both Prohorov metric and $W_2$-metric. Since $\G(\sE)\subseteq \sE$, by Lemma \ref{lem:cK}-(i), $\{f_{\m_n}\}$ is bounded in
$\hL^2(\hR^d)$. Thus by Banach-Alaoglu Theorem and noting that $\hL^2$ is reflexive, $\{f_{\m_n}\}$ is weakly compact, that is, there exists a subsequence $\{f_{\m_{n_k}}\}$ such that $ f_{\m_{n_k}}\limw f_\m$  in  $\hL^2(\hR^d)$, as $k\to\infty$. But this amounts to 
saying the $\m_{n_k} \Rightarrow \m$ in Prohorov metric. This, together with Lemma \ref{lem:cK}-(ii) and  the relationship between Wasserstein distance and weak convergence (see, e.g., \cite[Theorem 7.12]{Villani}), leads to that $\lim_{k\to\infty}\m_{n_k}=\m$
in $\sP_2(\hR^d)$, proving (ii). 

\ms

(iii) Let us assume that $\{\m_n\}\subset \sE$ such that $\m_n\Rightarrow \m$ in Prohorov metric.
The stability result in Proposition \ref{stable} shows that  $\rho^{\m_n}=\G_1(\m_n)\to \G_1(\m)=\rho^\m\in  \cS$ in $ \hL^1(\hR^d)$. Next, we show that $\G_2(\rho^{\m_n})\Rightarrow \G_2(\rho^\m)$ in Prohorov metric. Recall the definition  of $\G_2$, we define a family of 
functionals on $\sE$: for each $k, n\in\hN$ and $\bar{\mu}\in\sE$, 
\bea
\label{Fpi}
\left\{\ba{lll}
\dis	F^k_n(\bar{\mu}) := k G(\bar{\mu})  +\int_{\hR^d} \log \rho^{\m_n}(y) \bar{\mu}(dy);\ms\\
\dis	F^k(\bar{\mu}): = k G(\bar{\mu})  +\int_{\hR^d} \log \rho^\m(y) \bar{\mu}(dy).
\ea\right.
\eea
Then $ \m_n' =\G_2(\rho^{\m_n}):=\argmin_{\bar{\mu}\in\sE}F^k_n(\bar{\mu})$ and $\m'=\G_2(\rho^{\m}):=\argmin_{\bar{\mu}\in\sE}F^k(\bar{\mu})$.

To show that the minimizers $\m'_n\Rightarrow \m'$, we shall invoke 
the notion of $\G$-convergence (cf. \cite{DalMaso}). To be more precise, to prove the sequence $\{F^k_n\}_{n\ge 1}$  $\G$-converge to 
$F^k$ as $n\to \infty$, it is sufficient to show
\bea
\label{GammaConv}
\left\{\ba{lll}
\mbox{For every sequence $\bar{\mu}_n \Rightarrow \bar{\mu}$, it holds that $\dis F^k(\bar{\mu})\leq \liminf_{n} F^k_n(\bar{\mu}_n)$}; \ms\\
\mbox{There exists a sequence $\bar{\mu}_n \Rightarrow\bar{\mu}$, such that $\dis F^k(\bar{\mu})\geq \limsup_{n} F^k_n(\bar{\mu}_n)$.}
\ea\right.
\eea     

Now, note that $G$ is convex and $\sE$ is compact under Prohorov metric, we see that both $\{F^k_n\}$ and $F^k$ are {\it coercive} (in the sense that there exists minimizing sequence in $\sE\subset \sP_2(\hR^d)$). Thus, in light of the $\G$-convergence result (see  \cite[Theorem 7.1]{DalMaso}), in order to show $\m_n'\Rightarrow \m'$, it suffices to check that $\{F^k_n\}$  $\Gamma$-converges to $F^k$,  $k\in\hN$. To see this, let  $\{\bar{\mu}_n\}\subset \sE$ such that 
$\bar{\mu}_n\Rightarrow \bar{\mu}$, we have
\beaa
\label{rhoconv}
&&\Big|\int_{\hR^d} (\log \rho^{\m_n}(y)\bar{\mu}_n(dy) - \int_{\hR^d}\log \rho^\m(y)\bar{\mu}(dy)\Big|  \nonumber\\
&\le&  \int_{\hR^d} |\log \rho^{\m_n}(y) - \log \rho^\m(y)|f_{\bar\m_n}(y)(dy) +\int_{\hR^d} |\log \rho^{\m}(y)||f_{\bar{\mu}_n}(y) - f_{\bar{\mu}}(y)|dy
=:I_1+I_2, 
\eeaa
where $I_1$ and $I_2$ are the two integrals on the right side above. We now prove that both $I_1$ and $I_2$ converge to 0, as $n\to\infty$. To see this, first note that by \eqref{rhopsi} we have 
\begin{eqnarray*}
I_1& \le& \int_{\hR_d}|\psi^{\m_n}(y)-\psi^\mu(y)|f_{\bar\m_n(y)}dy + \int_{\hR_d}|\log f_{\m_n}(y) - \log f_{\m}(y)|f_{\bar\m_n(y)}dy\\
&\le& W_2(\m_n,\m)\int_{\hR_d}f_{\bar\m_n(y)}dy + K\int_{\hR_d}|\log f_{\m_n}(y) - \log f_{\m}(y)| g(y)dy =: I_{11}+I_{12},
\end{eqnarray*}
where the last inequality is due to \eqref{represent} and Assumption \ref{assum:g}. Since $\m_n\Rightarrow \m$,  again by Lemma \ref{lem:cK}-(ii) and \cite[Theorem 7.12]{Villani}, we have  $W_2(\mu_n,\m)\to 0$. Hence $I_{11} \to 0$ as $n\to \infty$. Furthermore, by Assumption \ref{assum:g}-(i) we see that, for  $\mu\in\sE$, one has $|\log f_\m(y)|\leq C|y|^2$ for some constant $C>0$,  that is,  $|\log f_{\m_n}(y) - \log f_{\m}(y)| g(y) \leq 2C|y|^2 g(y)$ for $y\in\hR^d$.  On the other hand, applying Lemma \ref{lem:cK}-(iii), we have $f_{{\mu}_n}\to f_{{\mu}}$ in 
$\hL^1$.  Thus, possibly along a subsequence (may assume itself, by virtue of Heine's Lemma),  it holds $\log f_{{\mu}_n}(y)\to \log f_{{\mu}}(y)$,  a.e. $y\in\hR^d$, thanks to the continuity of $\log$ function. Therefore, by Assumption \ref{assum:g} and Dominated Convergence Theorem, we see that $I_{12}\to 0$ as $n\to \infty$, whence $I_{1}\to 0$ as $n\to \infty$. Similarly, by \eqref{rhopsi}, \eqref{estpsi}, and Assumption \ref{assum:g}-(i), we have
\begin{eqnarray*}
I_2 &\leq& \int_{\hR^d}|\psi^{\mu}(y)||f_{\bar\m_n}(y)-f_{\bar\m}(y)|dy + \int_{\hR^d}|\log f_{\mu}(y)||f_{\bar\m_n}(y)-f_{\bar\m}(y)|dy\\
&\leq& 2C\int_{\hR^d}|y|^2|f_{\bar\m_n}(y)-f_{\bar\m}(y)|dy \to 0, \q \mbox{as $n\to \infty$}, 
\end{eqnarray*}
again thanks to Assumption \ref{assum:g}-(ii) and the Dominated Convergence Theorem. Finally, since $G(\cd)$ is continuous on $\sP_2(\hR^d)$, by definition (\ref{Fpi}) we see that $F^k_n(\bar{\mu}_n)\to F^k(\bar{\mu})$ whenever $\bar{\mu}_n\Rightarrow \bar{\mu}$. Thus,  by \eqref{GammaConv}, $\{F^k_n\}$ $\G$-converges to $F^k$, as $n\to\infty$. The proof is now complete.
\end{proof}

Finally, we shall establish an 
analogue of Theorem \ref{thm:convergence_opt_policy} in the case of general $\mi\in \sP_2(\hR^d)$. For technical convenience, in what follows we shall make use of the following extra assumptions to facilitate our discussion. Recall the function $\xi$ and $\eta$ defined by (\ref{est_potential}) and (\ref{cond_mu0}), respectively. 
\begin{Assumption}
\label{assump6}
(i) The penalty function $G$ satisfies  $G(\mu;\mu_{\rm tar})\geq W_2(\mu,\mu_{\rm tar})$;

\ms

(ii) In Assumption \ref{assump:general penalty term}-(ii), the function $\phi$ satisfies $\|\phi(\cd) e^{\xi(\cd)}\|_\infty<\infty$;

\ms 

(iii) For any $R>0$, there exists $M_R>0$, such that 
\bea
\label{assum:eta}
\hE[\eta^2(X^{t,x}_T)]=\int_{\hR^d} \eta^2(y)p(T, y; t,x) dy\le M_R, \qq (t,x)\in[0,T]\times B_R, 
\eea
where $B_R:=\{x\in\hR^d: |x|\le R\}$. 
\qed
\end{Assumption}
\begin{Remark}
{\rm (1) Assumption \ref{assump6}-(i) is not overly restrictive, and can be justified by Example \ref{G2}. 

\ms
\no(2) Note that the function $\phi$ in Assumption \ref{assump:general penalty term}-(ii) satisfies $\phi(y) e^{\l |y|^2}\le C$, for some $c>0$ and  $\xi(y)\sim c|y|^2$, as $|y|\to\infty$, Assumption  \ref{assump6}-(ii) amounts to saying that $\phi$ and $\xi$ are compatible. 

\ms
\no(3) While Assumption \ref{assump6}-(iii) is slightly stronger than the requirement (\ref{cond_mu0}), it would be trivial if the mapping $(t,x)\mapsto \hE[\eta^2(X^{t,x}_T)]$ is continuous, which is by no means stringent. 
\qed}
\end{Remark}

Our last result of this section is the following.
\begin{Theorem}
\label{thm:convergence_rate_general_initial}
Assume that the Assumptions \ref{assump:general penalty term}, \ref{assum:g} and \ref{assump6} are in force. Let $\mi\in \sE$ be given, and  
let $\widehat\alpha(t,x)$ and $\widehat\alpha^{k}(t,x)$, $(t,x)\in[0,T]\times\hR^d$ are optimal controls given in Proposition \ref{existak} and Lemma \ref{lemma:3.1}, respectively. Then, for any $R>0$, there exists $C_R>0$, such that for any $k\in\hN$, it holds that
\bea
\label{alphaest}
\int_0^T |\widehat\alpha^k(t,x) - \widehat\alpha(t,x)| dt   \leq \frac{C_R} {k}, \qq (t,x)\in[0,T]\times B_R. 
\eea
\end{Theorem}

\begin{proof}

We begin by denoting $\mu_k:=\hP_{X_T^{\widehat\a^k}}$, $\mt:=\hP_{X_T^{\widehat\a}}$, and let $\rho^{\m_k}=\G_1(\m_k)$, $k\in\hN$, 
$\rho^\mt=\G_1(\mt)$, respectively, as we defined before. Next, applying  \eqref{rhopsi}, we have
\begin{align}\begin{split}\label{eq62}
	|\rho^{\m_k}(y)-\rho^\mt(y)| &= |e^{\psi^{\m_k}(y)}f_{\m_k}(y)- e^{\psi^{\mt}(y)}f_{\rm tar}(y)| \\
	& \leq e^{\psi^{\mt}(y)} |f_{\m_k}(y) - f_{\rm tar}(y)| + f_{\m_k}(y) | e^{\psi^{\m_k}(y)} - e^{\psi^{\mt}(y)}|.
\end{split}
\end{align}

Let us now recall some facts from {Section \ref{sec:conv_opt_policy}}. First, note that the optimality of $\m_k$ implies that $G(\m_k)\leq \frac{C}{k}$ (cf. (\ref{Gbound})), for some generic constant $C>0$ independent of $k$, which we shall allow to vary from line to line below. Thus, by virtue of Assumption \ref{assump:general penalty term}-(ii), we can write 
$$|f_{\m_k}(y) - f_{\rm tar}(y)|\leq \frac{C}{k}\phi(y),  \qq y\in\hR^d,
$$  
where $\phi(y)e^{\xi(y)}\le C$, $y\in\hR^d$, thanks to Assumption \ref{assump6}-(ii).  Furthermore, under Assumption \ref{assump6}-(i), we can assume without loss of generality that the Schr\"odinger potentials $\psi^{\m_k}$ and $\psi^{\mt}$ all satisfy estimates \eqref{represent} and \eqref{est_potential}. Consequently, by Assumption \ref{assump6}-(i) and an easy application of Lemma \ref{psicontrol} and Newton-Leibniz formula we have
\begin{align*}
| e^{\psi^{\mt}(y)} - e^{\psi^{\m_k}(y)}|&= | {\psi^{\mt}(y)} - {\psi^{\m_k}(y)}|\int_0^1 \exp\{\psi^\mt(y) + \theta (\psi^{\m_k}(y) -\psi^\mt(y)\} d\theta \\
&\leq Ce^{\xi(y)}  W_2(\m_k,\mt) \leq C e^{\xi(y)} G(\m_k) \leq \frac{C}{k}e^{\xi(y)},  \q y\in\hR^d.
\end{align*}
Summarizing above and recalling  Assumption \ref{assum:g} and \eqref{cond_mu0},  we derive from (\ref{eq62}) that
\begin{align*}
|\rho^{\m_k}(y)-\rho^\mt(y)|\leq  \frac{C}{k}\left(e^{\psi^{\mt}(y)}\phi(y) + f_{\m_k}(y) e^{\xi(y)}\right)\leq   \frac{C}{k}\left(e^{\xi(y)}\phi(y) + g^2(y) e^{\xi(y)}\right)\leq  \frac{C}{k}(1 + \eta(y)),
\end{align*}
and therefore, given $R>0$, and $(t,x)\in [0,T]\times B_R$, we apply Assumption \ref{assump6}-(iii) to get 
\bea
\label{estrho}
\hE[|\rho^{\m_k} (X^{t,x}_T)- \rho^\mt(X^{t,x}_T)|]\le \hE[|\rho^{\m_k} (X^{t,x}_T)- \rho^\mt(X^{t,x}_T)|^2]^{\frac12}\leq \frac{C_R}{k}, 
\eea
where $C_R>0$ is some constant depending on the generic constant $C$ above and $M_R$ in (\ref{assum:eta}).

To complete the proof, let us recall that optimal strategies are of the form $\widehat\alpha^k(t,x)=\nabla \log h^{\m_k}(t,x)$, $k\in\hN$, and $ \widehat\alpha(t,x)=\nabla \log h(t,x)$, and $h^{\m_k}(t,x)$ and $h(t,x)$ are the solutions to the respective PDEs:
\bea
\label{hpde}
\begin{cases}
\partial_t h^{\m_k}  (t,x)+\scL_t h^{\m_k} (t,x)=0;\\
h^{\m_k} (T,x) =\rho^{\m_k}(x).
\end{cases}
\quad \quad
\begin{cases}
\partial_t h  (t,x)+ \scL_t h (t,x)=0;\\
h(T,x) = \rho^\mt(x),
\end{cases}
\eea
Furthermore, noting that $h^{\m_k}(t,x)=\hE[\rho^{\m_k}(X^{t,x}_T)]$, $h(t,x)=\hE[\rho^{\mt}(X^{t,x}_T)]$, and  by the Bismut-Elworthy-Li formula we have
\begin{equation*}
\nabla h^{\m_k}(t,x) = \hE[\rho^{\m_k}(X^{t,x}_T)N^{t,x}_T ]; \quad \nabla h(t,x) = \hE[\rho^{\mt}(X^{t,x}_T)N^{t,x}_T].
\end{equation*}
Thus, we have whenever $(t,x)\in[0,T]\times B_R$,
$$|\nabla h(t,x)| \leq K\hE[\eta^2(X^{t,x}_T)]^{\frac{1}{2} }\hE[|N^{t,x}_T|]^{\frac{1}{2}}\leq \frac{C_R}{\sqrt{T-t}},$$  and a similar argument as in \eqref{nablahbound} and \eqref{hkhbound}, together with the estimate  \eqref{estrho}, leads to that
\begin{align}\label{eq63}
|\nabla h^{\m_k}(t,x) - \nabla h(t,x)|\leq \frac{C_R}{k\sqrt{T-t}}; \quad |h^{\m_k}(t,x)-h(t,x)| \leq \frac{C_R}{k}, \q (t,x)\in[0,T]\times B_R.
\end{align}

Finally, we note that by definition the function $h$ is positive everywhere, and being the classical solution to the PDE (\ref{hpde}) it is continuous. Thus, given $R>0$, there exists $\d_R$>0, such that  $h(t,x)\ge \d_R$, for all $(t,x)\in[0,T]\times B_R$. Since \eqref{eq63} implies that $h^{\m_k}$ converges to $h$ uniformly on compacts, thus it must hold that
$h^{\m_k}(t,x)\geq\delta_R/2$, for $(t,x)\in[0,T]\times B_R$, and $k$  large enough. We therefore conclude, 
similar to \eqref{estalpha}, that
\begin{align*}
| \widehat\alpha^k(t,x) - \widehat\alpha(t,x)| 
&\leq \left| \frac{\nabla h^{\m_k}(t,x) - \nabla h(t,x) }{h^{\m_k}(t,x)} \right| + |\nabla h(t,x)| \left | \frac{h(t,x)  -  h^{\m_k}(t,x)}{h^{\m_k}(t,x)h(t,x)} \right|\le
\frac{C_R}{k\sqrt{T-t}}, 
\end{align*}
as $k\to\infty$,  for $(t,x)\in [0,T]\times B_R$, where $C_R$ depends on $M_R$ and $\d_R$ above, but independent of $k$. Integrating in $t$ we obtain (\ref{alphaest}). 
\end{proof}

\section{Conclusion}

We study the soft-constrained Schrödinger bridge problem (SCSBP) as a flexible alternative to the classical formulation for generative modeling. By replacing hard terminal constraints with a general penalty function, the SCSBP potentially offers greater flexibility and stability for generative AI tasks. Moreover, we establish linear convergence of both the value functions and the optimal controls as the penalty parameter tends to infinity, thereby providing a theoretical guarantee for the framework.

In future work, we will develop efficient algorithms for learning the SCSBP solutions and test the performance on benchmark generative AI tasks. This will allow us to translate the theoretical framework into practical tools, further demonstrating the potential of regularized stochastic control formulations for modern generative modeling.
\bibliographystyle{abbrv}
\bibliography{references}

\end{document}